\documentclass[journal,romanappendices]{IEEEtran}

\newif\ifieeesubmitted
\ieeesubmittedfalse

\usepackage{amsmath,amsfonts,amssymb,color}
\usepackage{amsthm}
\usepackage{algorithm}
\usepackage{algorithmic}
\usepackage{booktabs}
\usepackage{array}
\usepackage{graphicx}
\usepackage{cite}
\usepackage{url}
\usepackage[T1]{fontenc}
\usepackage{setspace}


\newtheorem{theorem}{Theorem}[section]
\newtheorem{lemma}[theorem]{Lemma}
\newtheorem{proposition}[theorem]{Proposition}
\newtheorem{corollary}[theorem]{Corollary}
\newtheorem{remark}[theorem]{Remark}
\newtheorem{assumption}[theorem]{Assumption}

\newcommand{\CVaR}{\operatorname{CVaR}}
\newcommand{\TV}{\operatorname{TV}}
\newcommand{\E}{\mathbb{E}}
\newcommand{\Pp}{\mathbb{P}}
\newcommand{\PP}{\mathbb{P}}
\newcommand{\QQ}{\mathbb{Q}}
\newcommand{\R}{\mathbb{R}}
\newcommand{\Ot}{\widetilde{\mathcal O}}

\newcommand{\cC}{\mathcal C}
\newcommand{\cP}{\mathcal P}

\newcommand{\mainref}[1]{\ref{#1}}
\newcommand{\mainequation}[1]{\eqref{#1}}
\DeclareMathOperator*{\argmin}{arg\,min}
\DeclareMathOperator*{\argmax}{arg\,max}

\begin{document}

\title{Robust Risk-Sensitive Reinforcement Learning from Corrupted Human Feedback}

\author{Xinyi~Ni and Lifeng~Lai%
\thanks{The authors are with the Department of Electrical and Computer Engineering, University of California, Davis, CA 95616, USA (e-mail: \{xni, lflai\}@ucdavis.edu). This work was supported in part by National Science Foundation under grants CCF-2232907 and ECCS-2448268. }%
\thanks{This work has been submitted to the IEEE for possible publication. Copyright may be transferred without notice, after which this version may no longer be accessible.}%
}

\maketitle

\begin{abstract}
	Reinforcement learning with human feedback (RLHF) learns from human comparisons, which can be corrupted or deliberately manipulated. This paper studies online risk-sensitive RLHF with static conditional value-at-risk (CVaR) under adversarial preference-label flips. We consider additive linear rewards and a fixed-reference protocol with one comparison per episode and at most $C$ flipped labels over $K$ episodes. We propose weighted streamed-preference CVaR RLHF (WSP-CVaR-RLHF), which combines uncertainty-weighted reward estimation with optimistic augmented-state CVaR planning. For known transitions and normalized rewards, we establish the regret bound
	\(
	\widetilde O\!\left(
	\frac{d}{\kappa}\sqrt{\frac{K}{\alpha}}
	+\frac{dC}{\kappa\alpha}
	\right),
	\)
	up to lower-order terms, where $d$ is the reward-feature dimension, $\alpha$ is the CVaR level, and $\kappa$ characterizes the preference link. The bound separates the clean statistical cost from the penalty caused by corrupted feedback. We further extend the analysis to unknown tabular transitions, \textcolor{black}{where the trajectory distribution entering the CVaR objective must be learned together with the reward. We address the resulting coupled uncertainty using rectangular transition confidence sets, joint optimistic planning, and a history-level CVaR simulation argument.} Experiments under four adversarial attacks demonstrate that WSP-CVaR-RLHF consistently reduces cumulative regret relative to its unweighted robust counterpart while preserving confidence-set coverage. 
\end{abstract}


\section{Introduction}

Reinforcement learning with human feedback (RLHF) provides reward information when the objective of a sequential decision-making problem is difficult to specify numerically.  In online RLHF, an agent improves its policy while collecting pairwise comparisons between trajectories ~\cite{christiano2017deep,ibarz2018reward,zhu2023principled,wang2024rlhf,ye2024online}. The quality of the learned policy then depends on the reliability of these comparisons.  Noisy, biased, or manipulated labels can distort the inferred reward, and the RL solver may further reinforce this error when optimizing the policy~\cite{bukharin2024robust,mandal2024corruption,yang2025human}.

This issue becomes more consequential when the objective emphasizes rare poor outcomes. Standard RL maximizes the expected return, which may hide severe losses occurring on a small fraction of trajectories. Risk-sensitive RL instead evaluates the return distribution. In this paper, we focus on static conditional value-at-risk (CVaR): for $\alpha\in(0,1]$, lower-tail CVaR measures the average return among the worst $\alpha$ fraction of complete trajectories. CVaR has been widely studied in RL~\cite{chow2014algorithms,chow2015risk,tamar2015optimizing,tamar2015policy,wang2023near,ni2024risk}. Recent work has extended risk-sensitive learning to human feedback through iterated CVaR~\cite{chen2023provably}, static and nested quantile objectives~\cite{zhao2025ra}, and a clean preference-to-reward interface for static CVaR~\cite{ni2026isit}. However, these results assume clean stochastic comparisons and do not characterize how adversarial preference corruption propagates into the tail-risk objective.

Static CVaR creates a distinct amplification mechanism.  \textcolor{black}{A corrupted comparison changes the reward estimate, which can alter both the policy ranking and the trajectories that form the lower tail.}  Equivalently, the dual form of CVaR reweights lower-tail trajectories by as much as $1/\alpha$.  Therefore, an error concentrated on these trajectories can have a much larger effect than the same error under an expected-return objective.  Existing robust RLHF methods study expected or ideal comparison utility, offline reward recovery, or repeated-query defenses~\cite{bukharin2024robust,mandal2024corruption,yang2025human,multisourceimperf2025}. An offline distributional method uses CVaR to model stochastic reward uncertainty~\cite{xu2024uncertainty}, but does not consider adversarial label corruption or online regret.  To the best of our knowledge, there is no prior work that provides a cumulative-regret guarantee for robust risk-sensitive RLHF with adversarially corrupted preference feedback.

Motivated by the clean fixed-reference protocol of~\cite{ni2026isit}, we consider a streaming setting in which the learner receives one comparison per episode. Each executed trajectory occurrence is compared once with a fixed feasible reference, without deliberately collecting repeated labels for that occurrence. This setting captures online calibration against an established reference when repeated evaluation of the same trajectory occurrence is unavailable. We assume the unknown reward is additive and linear in known stage features, and the clean comparison follows a generalized linear preference model. After observing the clean label, an adversary may flip it, subject to a pathwise budget of at most $C$ flips over $K$ episodes. \textcolor{black}{The learner observes only the corrupted labels.}

\textbf{Our contribution}: \textcolor{black}{We aim to design robust risk-sensitive RLHF algorithms to mitigate potential impacts of adversarial manipulations. The main idea of our design is to control the influence of adversarial flips at the reward-estimation stage before the resulting uncertainty is propagated through CVaR planning. \textcolor{black}{Specifically, a corrupted label can have a large influence on reward estimation when its comparison feature lies in a poorly estimated direction. We therefore downweight comparisons with greater uncertainty to cap their self-normalized contributions to the estimating equation. The weights are determined before the current label is generated, and the resulting influence bound allows the aggregate effect of at most $C$ flips to be incorporated into the reward confidence radius.} Based on this idea, we propose \emph{WSP-CVaR-RLHF}, a weighted streamed-preference algorithm that couples the resulting reward confidence set with optimistic CVaR planning.} The uncertainty-weighted estimator is inspired by the clipped-confidence method of~\cite{di2025dueling}. Unlike their dueling-bandit analysis, our decision error must be propagated through the CVaR of the return distribution induced by the selected policy.

\textcolor{black}{We first derive a regret bound of WSP-CVaR-RLHF for the case with known transition kernels and normalized additive linear rewards.}
The leading terms in our regret bound are
\(\widetilde O\!\left( \underbrace{\frac{d}{\kappa}\sqrt{\frac{K}{\alpha}}}_{\text{clean statistical term}} + \underbrace{\frac{dC}{\kappa\alpha}}_{\text{corruption term}} \right).
\)
Here $K$ is the number of episodes, $C$ is the supplied upper bound on the realized flip budget, $\alpha$ is the risk level in CVaR, $d$ is the reward-feature dimension, and $\kappa$ characterizes the preference link. The different dependence on $\alpha$ follows from the CVaR envelope: it has total mass one, while its weight on an individual lower-tail trajectory can reach $1/\alpha$. \textcolor{black}{The corruption penalty grows linearly with $C$. When $C=0$, WSP uses unit weights, and our bound reduces to $\widetilde O\!\left(\frac{d}{\kappa}\sqrt{\frac{K}{\alpha}}\right)$ up to lower-order terms. At $\alpha=1$, static CVaR reduces to the expected return and the leading term has the same dependence on $d$, $\kappa$, and $K$ as the clean contextual dueling-bandit guarantee of~\cite{di2025dueling}.}

We further consider the unknown tabular transition-kernel scenario. 
For this case, our algorithm combines the reward confidence set with rectangular transition confidence sets, and a history-level CVaR simulation lemma separates reward-estimation error from transition-estimation error. Under the normalized reward scale, the resulting regret bound retains the known-transition reward terms and adds a transition-learning term of order $\widetilde O\!\left((H/\alpha)(SA+S\sqrt{AK})\right)$. Finally, experiments with both known- and unknown transition kernels under four adversarial feedback attacks demonstrate that WSP-CVaR-RLHF consistently reduces cumulative regret relative to its unweighted robust counterpart while preserving confidence-set coverage.

\textbf{Related work:} Existing \textit{risk-sensitive RLHF} methods study clean human feedback under iterated CVaR~\cite{chen2023provably}, static or nested quantile risk~\cite{zhao2025ra}, and a fixed-reference static-CVaR interface~\cite{ni2026isit}, or optimize risk-aware offline alignment objectives~\cite{zhang2025radpo,chittepu2026safe,horwitz2026risk}. None considers adversarially corrupted preference feedback. Compared with the closest static-CVaR framework~\cite{ni2026isit}, we study one comparison per episode under a pathwise flip budget and cumulative regret, rather than clean preference-to-reward sample and query complexity. For \textit{adversarial preference learning}, corrupted dueling-bandit methods~\cite{di2025dueling,cheng2025imperfect} optimize pairwise utility without an MDP-induced return distribution or a tail-risk objective. Adversarial preferences in tabular MDPs~\cite{tsuchiya2025adversarial} use a varying Borda-score objective, while online adversarial RLHF~\cite{yang2025human} uses expected return and repeated queries under a per-trajectory attack budget. Multi-source mismatch~\cite{multisourceimperf2025} and Bayesian best-policy identification~\cite{agnihotri2026best} also differ from the present single-source realized-flip model. Our analysis instead propagates reward uncertainty through static CVaR, whose lower-tail envelope produces different statistical and corruption dependence on $\alpha$. Robust learning from corrupted offline preferences is studied in~\cite{bukharin2024robust,mandal2024corruption,prefflip2025,chen2025cov}, while adversarial online CVaR control with numerical losses is considered in~\cite{chen2026adversarialcvar}. These works do not address cumulative static-CVaR regret for robust risk-sensitive RLHF with an online stream of corrupted feedback.

\textbf{Outline:} The remainder of this paper is organized as follows. Section~\ref{sec:prelim} reviews episodic MDPs and static CVaR. Section~\ref{sec:model} formulates the preference-feedback model, corruption protocol, and problem statement. Section~\ref{sec:algorithm} presents WSP-CVaR-RLHF. \textcolor{black}{Sections~\ref{sec:main} and~\ref{sec:unknown} establish the regret guarantees for known and unknown transitions, respectively.} Section~\ref{sec:experiments} presents the numerical results, and Section~\ref{sec:conclusion} offers concluding remarks.
\section{Preliminaries}
\label{sec:prelim}

\subsection{Episodic Markov Decision Processes}

We consider an episodic MDP $(\mathcal S,\mathcal A,H,K,\mathbb P^\star,r^\star)$, where $|\mathcal S|=S$, $|\mathcal A|=A$, $H$ is the episode length, $K$ is the number of episodes, $\mathbb P_h^\star(\cdot\mid s,a)$ is the transition kernel at step $h$, and \textcolor{black}{$r^\star$ is the unknown deterministic reward model specified in Section~\ref{sec:model}.} Let $\Pi_{\mathcal H}$ be the class of history-dependent policies. Each $\pi\in\Pi_{\mathcal H}$ consists of decision rules $\{\pi_h\}_{h=1}^H$, where $\pi_h:\mathcal S\times\mathcal H_h\to\Delta(\mathcal A)$, $\mathcal H_h=(\mathcal S\times\mathcal A)^{h-1}$ denotes the history before the current state at step $h$, and $\Delta(\mathcal A)$ is the probability simplex over $\mathcal A$. Every episode starts from a fixed known state $s_1$. At step $h\in[H]:=\{1,\ldots,H\}$, the learner selects $a_h$ and observes
$
s_{h+1}\sim\mathbb P_h^\star(\cdot\mid s_h,a_h).
$
The resulting trajectory is
$
\tau=(s_1,a_1,s_2,a_2,\ldots,s_H,a_H,s_{H+1}).
$
We use $\mathcal T$ to denote the set of all syntactically admissible length-$H$ trajectories.

\subsection{Static CVaR}
We adopt conditional value-at-risk (CVaR), a coherent and widely used risk measure in risk-sensitive RL~\cite{prashanth2014policy,chow2014algorithms,tamar2015optimizing,tamar2015policy,du2022provably,ni2024risk,wang2023near,zhang2024cvar}. For a random variable $X$ and risk level $\alpha\in(0,1]$,
\begin{equation}
\CVaR_\alpha(X)
=\sup_{b\in\R}\left\{b-\frac1\alpha\E[(b-X)^+]\right\}, 
\label{eq:cvar}
\end{equation}
where $(x)^+=\max\{x,0\}$. As a coherent risk measure, CVaR is monotone, translation equivariant, positively homogeneous, and superadditive~\cite{rockafellar2000optimization,artzner1999coherent,acerbi2002coherence}. It also admits the risk-envelope~\cite{ni2024robust}
\begin{equation}
\CVaR_\alpha(X)
=\inf_{q:\,0\le q\le1/\alpha,\,\E[q]=1}\E[qX].
\label{eq:dual}
\end{equation}
\section{Model and Problem Formulation}
\label{sec:model}

We now specify the reward model, corrupted-feedback protocol, and static-CVaR objective. The learner does not observe numerical rewards during interaction and receives reward information only through trajectory comparisons.

\subsection{Reward Model and Preference Feedback}
\label{sec:preference-model}

We use the standard random-utility comparison model~\cite{christiano2017deep,zhan2023provable,wang2024rlhf}: before corruption, a comparison between trajectories $\tau$ and $\tau'$ is Bernoulli with mean \(\sigma\!\left(r^\star(\tau)-r^\star(\tau')\right),\) where $\sigma$ is a known link function and $r^\star$ is the underlying reward that guides human preferences.

\ifdefined\submissionversion
\begin{figure}[t]
\centering
\includegraphics[width=0.60\linewidth]{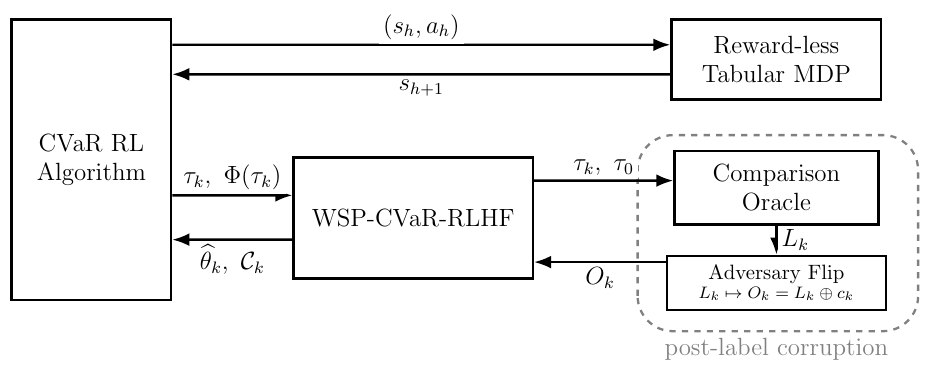}
\else
\begin{figure*}[t]
\centering
\includegraphics[width=0.52\textwidth]{figures/wsp_cvar_rlhf_protocol.pdf}
\fi
\caption{Interaction protocol for WSP-CVaR-RLHF. The weighted preference estimator processes possibly flipped comparison labels and supplies a reward estimate and confidence set to the static-CVaR planner.}
\label{fig:protocol}
\ifdefined\submissionversion
\end{figure}
\else
\end{figure*}
\fi

For the underlying reward, we adopt the commonly used additive linear reward model~\cite{sadigh2017active,biyik2018batch,jin2020provably,chen2023provably,wang2024rlhf,di2025dueling}.
\begin{assumption}[Additive linear reward]
	\label{ass:model}
	For each $h\in[H]$, let
	$\phi_h:\mathcal S\times\mathcal A\to\R^d$ be a known feature map,
	and let $\Theta\subset\R^d$ be a known compact convex parameter set.
	For every $\theta\in\Theta$, define
	\begin{equation}
		r_{\theta,h}(s,a)=\langle\theta,\phi_h(s,a)\rangle,
		\qquad
		r_\theta(\tau)=\sum_{h=1}^H r_{\theta,h}(s_h,a_h).
		\label{eq:linear-reward}
	\end{equation}
		The underlying true reward that generates the human comparisons is
		realizable in this class: there exists $\theta^\star\in\Theta$ such that
		$r^\star(\tau)=r_{\theta^\star}(\tau)$ for every $\tau\in\mathcal T$.
	
	Moreover, there is a known constant $B>0$ such that
	$\|\theta\|_2\le B$ for every $\theta\in\Theta$.
	Fix a feasible reference trajectory $\tau_0$ and define
	\begin{equation}
		z(\tau)=\sum_{h=1}^H
		\bigl(\phi_h(s_h,a_h)-\phi_h(s_{0,h},a_{0,h})\bigr).
		\label{eq:centered-feature}
	\end{equation}
	For every $\tau\in\mathcal T$, $\|z(\tau)\|_2\le1$.
\end{assumption}

\begin{remark}
	Every deterministic stage-dependent tabular reward is represented, up to the normalization below, by taking $d=HSA$ and using one-hot features indexed by $(h,s,a)$.
	The normalization $\sup_{\tau\in\mathcal T}\|z(\tau)\|_2\le1$ is without loss of generality because any finite feature radius can be absorbed into the parameter norm bound. Thus, $B=O(1)$ in the simplified rates refers to normalized whole-trajectory rewards and need not hold under per-stage normalization.
\end{remark}

We further define the reference-centered stage and trajectory rewards by
\begin{align}
	r_{\theta,h}^{\rm ctr}(s,a)
	&:=\langle\theta,\phi_h(s,a)-\phi_h(s_{0,h},a_{0,h})\rangle,
	\label{eq:centered-stage-reward}\\
	r_{\theta,\tau_0}(\tau)
	&:=\sum_{h=1}^H r_{\theta,h}^{\rm ctr}(s_h,a_h)
	=\langle\theta,z(\tau)\rangle.
	\label{eq:centered-reward}
\end{align}

In episode $k$, the executed trajectory $\tau_k$ is compared with a fixed reference
trajectory $\tau_0$, and the clean label satisfies
	\begin{equation}
	L_k\mid\tau_k\sim\operatorname{Ber}\!\left(\sigma(z(\tau_k)^\top\theta^\star)\right).
	\label{eq:clean-label}
	\end{equation}
Here $L_k=1$ means that $\tau_k$ is preferred to $\tau_0$. Unlike the clean CVaR-HF-P2R protocol in~\cite{ni2026isit}, the observed bit below may be flipped and only one comparison is requested in each episode.

\begin{assumption}[Preference link]
	\label{ass:link}
	The known link $\sigma:\R\to[0,1]$ is differentiable and strictly increasing, satisfies $\sigma(0)=1/2$ and $\sigma(x)+\sigma(-x)=1$, and obeys
	\begin{equation}
		0<\kappa\le\sigma'(z(\tau)^\top\theta)\le1
		\label{eq:link-curvature}
	\end{equation}
	for every $\tau\in\mathcal T$ and $\theta\in\Theta$.
\end{assumption}

The symmetry condition gives order-consistent comparisons. The derivative lower bound is the standard local identifiability condition used in generalized linear bandits and preference-based RLHF~\cite{filippi2010parametric,li2017provably,chen2023provably,wang2024rlhf}. It is required only on the attainable score set $\{z(\tau)^\top\theta:\tau\in\mathcal T,\theta\in\Theta\}$. For the logistic link, this condition holds whenever the attainable score set is bounded.

\subsection{Adversarial Post-Label Corruption}

After the clean label is generated, a strong adversary observes $L_k$ and chooses a flip indicator $c_k\in\{0,1\}$. The observed label $O_k$ of the learner is defined as:
\begin{equation}
	O_k=L_k\oplus c_k,
	\qquad
	\sum_{k=1}^Kc_k\le C
	\quad\text{pathwise},
	\label{eq:corruption-model}
\end{equation}
where $\oplus$ denotes exclusive OR. The learner is given an integer upper bound $C\in\{0,\ldots,K\}$ on the realized flip budget. Because the protocol requests only one label for each executed trajectory occurrence, the learner cannot remove corruption by deliberately repeating that query.

\textbf{Information structure:} For each episode $k\in[K]$, let $\mathcal F_{k-1}$ denote the learner's observable history before the episode and let $\mathcal M_{k-1}$ be the corresponding master history, which additionally contains previous latent clean labels, flip indicators, and adversary randomization. The selected physical policy is $\mathcal F_{k-1}$-measurable. After the trajectory is observed, the weight $w_k$ is fixed before the clean label $L_k$ is generated and before the adversary chooses $c_k$. We adopt the fresh-randomness convention: conditional on $\mathcal F_{k-1}$ and the selected policy, the episode-$k$ transition innovations are independent of $\mathcal M_{k-1}$, and $L_k$ is generated from fresh Bernoulli randomness conditional on the resulting trajectory. Consequently, the conditional trajectory law given $\mathcal M_{k-1}$ equals that given $\mathcal F_{k-1}$. Appendix~\ref{app:filtration} provides the formal construction.

\subsection{Static-CVaR Objective and Augmented-State Policies}

For a transition kernel $\PP$, reward parameter $\theta$, and any fixed physical policy $\pi\in\Pi_{\mathcal H}$, define
\begin{equation}
J_{\PP,\theta}(\pi)
=\CVaR_\alpha^{\PP,\pi}(r_{\theta,\tau_0}(\tau)).
\label{eq:policy-objective}
\end{equation}

Assumption~\ref{ass:model} implies $|r_{\theta,\tau_0}(\tau)|\le B$, so an optimal threshold may be restricted to $[-B,B]$~\cite{ni2024risk}. Following~\cite{bauerle2011markov}, additive rewards allow CVaR optimization on the augmented state space $\mathcal S^{\rm Aug}=\mathcal S\times\R$. Let $\Pi^{\rm Aug}$ denote the class of deterministic Markov policies $\rho=\{\rho_h\}_{h=1}^H$, where $\rho_h:\mathcal S\times\mathbb R\to\mathcal A$. Starting from $(s_1,b_1)$ with $b_1\in[-B,B]$, the agent chooses $a_h=\rho_h(s_h,b_h)$, transitions according to $\PP_h(\cdot\mid s_h,a_h)$, and updates
\(
b_{h+1}=b_h-r_{\theta,h}^{\rm ctr}(s_h,a_h).
\)
For $\rho\in\Pi^{\rm Aug}$, define the expected terminal shortfall from augmented state $(s_h,b_h)$ by
\begin{equation}
V_{h,\PP,\theta}^{\rho}(s_h,b_h)
=\E_{\PP,\rho}\!\left[
\left(b_h-\sum_{t=h}^H r_{\theta,t}^{\rm ctr}(s_t,a_t)\right)^+
\middle|s_h,b_h
\right].
\label{eq:augmented-value}
\end{equation}
For every fixed initial threshold $b_1$, the expected terminal shortfall is linear in each randomized action distribution. Hence randomization cannot improve the minimum, and standard backward induction attains the optimum with a deterministic Markov policy in $\Pi^{\rm Aug}$. Consequently, the optimal CVaR value satisfies
\begin{equation}
\ifdefined\submissionversion
	J_{\PP,\theta}^\star:=\sup_{\pi\in\Pi_{\mathcal H}}J_{\PP,\theta}(\pi)=\max_{b_1\in[-B,B]}\left\{b_1-\frac1\alpha\min_{\rho\in\Pi^{\rm Aug}}V_{1,\PP,\theta}^{\rho}(s_1,b_1)\right\}.
	\else
	\begin{aligned}
	J_{\PP,\theta}^\star
	&:=\sup_{\pi\in\Pi_{\mathcal H}}J_{\PP,\theta}(\pi)\\
&=\max_{b_1\in[-B,B]}
\left\{b_1-\frac1\alpha
\min_{\rho\in\Pi^{\rm Aug}}
	V_{1,\PP,\theta}^{\rho}(s_1,b_1)\right\}.
	\end{aligned}
	\fi
	\label{eq:augmented-cvar}
\end{equation}
For any $\vartheta\in\Theta$, $b\in[-B,B]$, and $\rho\in\Pi^{\rm Aug}$, let $\pi^{\rho,b,\vartheta}\in\Pi_{\mathcal H}$ denote the induced physical policy. Specifically,
\begin{equation}
\pi_h^{\rho,b,\vartheta}(\cdot\mid s_h,\mathcal H_h)
=\delta_{\rho_h(s_h,b_h^{\vartheta})},
\qquad
b_h^{\vartheta}=b-\sum_{t=1}^{h-1}r_{\vartheta,t}^{\rm ctr}(s_t,a_t),
\label{eq:induced-policy}
\end{equation}
where $\delta_a$ is the point mass at action $a$. Thus, $\rho$ is Markov on the augmented state, while $\pi^{\rho,b,\vartheta}$ is the corresponding executable policy on physical histories. For fixed $(\rho,b,\vartheta)$, the trajectory law of $\pi^{\rho,b,\vartheta}$ remains unchanged when the evaluation parameter in $J_{\PP,\theta}$ changes. In particular, if $(b^\star,\rho^\star)$ attains Eq.~\eqref{eq:augmented-cvar}, then $J_{\PP,\theta}^\star=J_{\PP,\theta}(\pi^{\rho^\star,b^\star,\theta})$.

\subsection{Regret Formulation}
\label{sec:regret-definition}

Since the reference trajectory $\tau_0$ is fixed before learning and reused in every comparison, $r_\theta(\tau_0)$ is constant across policies for every fixed $\theta$. Thus, translation equivariance of static CVaR yields, for every policy $\pi$, \(\CVaR_\alpha^\pi(r_{\theta,\tau_0})=\CVaR_\alpha^\pi(r_\theta)-r_\theta(\tau_0).\) Therefore, for any two policies $\pi$ and $\pi'$, the policy-value difference satisfies
\ifdefined\submissionversion
\[
\CVaR_\alpha^\pi(r_{\theta,\tau_0})-\CVaR_\alpha^{\pi'}(r_{\theta,\tau_0})=\CVaR_\alpha^\pi(r_\theta)-\CVaR_\alpha^{\pi'}(r_\theta).
\]
\else
\[
\begin{aligned}
	&\CVaR_\alpha^\pi(r_{\theta,\tau_0})
	-\CVaR_\alpha^{\pi'}(r_{\theta,\tau_0})\\
	&\qquad=\CVaR_\alpha^\pi(r_\theta)-\CVaR_\alpha^{\pi'}(r_\theta).
\end{aligned}
\]
\fi

\textcolor{black}{In each episode $k\in[K]$, a learning algorithm uses the observable history of possibly corrupted comparisons and past transitions to select a physical policy $\pi_k\in\Pi_{\mathcal H}$. The objective is to learn from this corrupted feedback while minimizing the performance loss relative to an optimal policy under the true reward and transition kernel.}
The static-CVaR regret over $K$ episodes is
\begin{equation}
\operatorname{Regret}(K)
=\sum_{k=1}^K
\left[J_{\PP^\star,\theta^\star}^\star
-J_{\PP^\star,\theta^\star}(\pi_k)\right].
\label{eq:regret}
\end{equation}
By Eq.~\eqref{eq:centered-reward} and translation equivariance, this is also the static-CVaR regret under true reward $r^\star$. 
\section{WSP-CVaR-RLHF Algorithm}
\label{sec:algorithm}
In this section, we introduce the proposed WSP-CVaR-RLHF algorithm. Before explaining detailed steps of
\textcolor{black}{Algorithm~\ref{alg:weighted} in the following subsections, we first summarize major steps of WSP-CVaR-RLHF:  1) In each episode, the method first constructs a reward confidence set from the weighted comparison history, then uses this set to select and execute an optimistic static-CVaR policy (Lines~\ref{line:wsp-estimation}--\ref{line:wsp-execution}); 2) Once the trajectory is observed, it assigns a predictable uncertainty weight before receiving the possibly flipped label and updating the design matrix (Lines~\ref{line:wsp-weighting}--\ref{line:wsp-update}). This order fixes the weight before the current label is generated and limits the influence of corrupted comparisons.} 

\begin{algorithm}[H]
	\caption{WSP-CVaR-RLHF}
	\label{alg:weighted}
	\begin{algorithmic}[1]
		\STATE \textbf{Input:} the model and feedback protocol in Section~\ref{sec:model}; \textcolor{black}{known $\PP^\star$;} ridge $\lambda>0$, clipping threshold $\chi>0$ (for $C>0$), confidence level $\delta_{\rm cs}\in(0,1)$, and risk level $\alpha\in(0,1]$.
		\STATE \textbf{Initialization}: $\Sigma_1=\lambda I_d$ and an empty preference stream.
		\FOR{$k=1,\ldots,K$}
		\STATE \label{line:wsp-estimation}\textbf{Weighted preference estimation:} compute
		$\widehat\theta_k$, $\beta_k$, and $\cC_k$ from
		Eqs.~\eqref{eq:mest}--\eqref{eq:conf}.
		\STATE \label{line:wsp-planning}\textbf{CVaR planning:} solve Eq.~\eqref{eq:knownoracle} for
		$(\widetilde\theta_k,\widetilde b_k,\rho_k)$.
		\STATE \label{line:wsp-execution}Form and execute $\pi_k=\pi^{\rho_k,\widetilde b_k,\widetilde\theta_k}$, and observe trajectory $\tau_k$ and feature
		$z_k=z(\tau_k)$.
		\STATE \label{line:wsp-weighting}\textbf{Uncertainty weighting:} set
		$u_k=\|z_k\|_{\Sigma_k^{-1}}$ and
		choose $w_k$ by Eq.~\eqref{eq:weight}; in particular, $w_k=1$ when $C=0$.
		\STATE \label{line:wsp-feedback}Receive the corrupted comparison bit $O_k$ generated by the protocol in
		Fig.~\ref{fig:protocol}. \textcolor{black}{Append $(z_k,O_k,w_k)$ to the preference stream.}
		\STATE \label{line:wsp-update}\textbf{Design update:} update
		$\Sigma_{k+1}=\Sigma_k+\kappa w_k z_kz_k^\top$.
		\ENDFOR
	\end{algorithmic}
\end{algorithm}

\subsection{Why Weight the Comparisons?}

A corrupted label has the greatest influence when its feature direction is poorly estimated.  WSP-CVaR-RLHF controls this influence in the local geometry of the reward estimator.  Fix a ridge parameter $\lambda>0$ and a clipping threshold $\chi>0$ before interaction. The clipping threshold may depend on the disclosed problem parameters but not on the observed data.  For each episode $k\in[K]$, write $z_k:=z(\tau_k)$ for its comparison feature, let $I_d$ denote the $d\times d$ identity matrix, and define the weighted design matrix \(\Sigma_k:=\lambda I_d+\kappa\sum_{i<k}w_i z_i z_i^\top.\)

Before observing the comparison label, the algorithm computes the uncertainty score
$u_k=\|z_k\|_{\Sigma_k^{-1}}.$
If $u_k$ is large, the comparison is clipped by the weight
\begin{equation}
w_k=
\begin{cases}
1,&C=0\text{ or }u_k=0,\\
\min\{1,\chi/u_k\},&C>0\text{ and }u_k>0.
\end{cases}
\label{eq:weight}
\end{equation}
For $C>0$, this ensures the pathwise inequality
\(
w_k\|z_k\|_{\Sigma_k^{-1}}\le \chi.
\)
Therefore, at most $C$ flipped labels can contribute $\chi C$ to the confidence radius. For $C=0$, the algorithm uses the explicit clean branch $w_k=1$ and does not introduce artificial clipping. The weighted estimator, clipping rule, and pathwise score control are adapted from~\cite{di2025dueling}. Related corruption-budget inflations appear in corruption-robust bandit and RL analyses~\cite{lykouris2021corruption,wei2022model}.

\subsection{Estimator and Confidence Set}

Let $\Psi$ be an antiderivative of the preference link function $\sigma$, so that $\Psi'=\sigma$.  At the beginning of episode $k$, define the constrained weighted M-estimator
\begin{align}
\widehat\theta_k\in\argmin_{\theta\in\Theta}\;&
\frac{\lambda}{2}\|\theta\|_2^2
+\sum_{i<k}w_i
\left[\Psi(z_i^\top\theta)-O_i z_i^\top\theta\right].
\label{eq:mest}
\end{align}
For the logistic link, this objective equals the regularized weighted Bernoulli negative log-likelihood. For a general increasing link, its gradient is $\sum_{i<k}w_i(\sigma(z_i^\top\theta)-O_i)z_i+\lambda\theta$. The confidence analysis uses this gradient together with the constrained first-order variational inequality, monotonicity of $\sigma$, and the lower derivative bound. Therefore, a Bernoulli likelihood interpretation is not required outside canonical links. This construction follows the general approach to generalized linear bandit and preference-based confidence sets~\cite{filippi2010parametric,li2017provably,chen2023provably,zhu2023principled}.

For a confidence-set failure probability $\delta_{\rm cs}\in(0,1)$, the confidence radius and set are
\begin{align}
\beta_k={}&\sqrt\lambda B+\chi C+
\frac{1}{\sqrt\kappa}
\sqrt{2\log\frac{\det(\Sigma_k)^{1/2}}
{\det(\lambda I_d)^{1/2}\delta_{\rm cs}}},
\label{eq:radius}\\
\cC_k={}&
\left\{\theta\in\Theta:
\|\theta-\widehat\theta_k\|_{\Sigma_k}\le \beta_k\right\}.
\label{eq:conf}
\end{align}
The three terms in $\beta_k$ account for ridge bias, adversarial corruption, and stochastic estimation error. The final term follows by applying the self-normalized concentration theorem of Abbasi-Yadkori \emph{et al.}~\cite[Theorem~1]{abbasi2011improved} to a rescaled predictable design; the corruption term follows pathwise from the clipping rule.

\subsection{\textcolor{black}{Optimistic CVaR Planning}}

When transitions are known, the algorithm combines the confidence set with the augmented-state CVaR formulation in Eq.~\eqref{eq:augmented-cvar}. It selects
\begin{equation}
\ifdefined\submissionversion
(\widetilde\theta_k,\widetilde b_k,\rho_k)\in\argmax_{\theta\in\cC_k,\,b\in[-B,B],\,\rho\in\Pi^{\rm Aug}}\left\{b-\frac1\alpha V_{1,\PP^\star,\theta}^{\rho}(s_1,b)\right\},
\else
(\widetilde\theta_k,\widetilde b_k,\rho_k)
\in\argmax_{\substack{\theta\in\cC_k,\,b\in[-B,B],\\
\rho\in\Pi^{\rm Aug}}}
\left\{b-\frac1\alpha V_{1,\PP^\star,\theta}^{\rho}(s_1,b)\right\},
\fi
\label{eq:knownoracle}
\end{equation}
with fixed measurable tie-breaking. The selected triple induces the physical policy
\begin{equation}
\pi_k:=\pi^{\rho_k,\widetilde b_k,\widetilde\theta_k},
\label{eq:episode-policy}
\end{equation}
which the learner executes. Thus, the internal budget starts from $\widetilde b_k$ and is updated using $\widetilde\theta_k$, while $\rho_k$ remains the augmented-state Markov rule. Since $\widehat\theta_k\in\cC_k$, the confidence set is nonempty, and the resulting trajectory produces the next preference query.

\textcolor{black}{For each fixed reward parameter and initial threshold, the Bellman recursion in Eq.~\eqref{eq:oracle-dp}, with $\PP=\PP^\star$, computes the minimum expected terminal shortfall using the centered stage rewards in Eq.~\eqref{eq:centered-stage-reward}. The outer maximization in Eq.~\eqref{eq:knownoracle} selects the reward parameter and threshold. The analysis assumes an exact planning oracle; Proposition~\ref{prop:oracle-wellposed} establishes attainment and measurable selection. For the finite known-transition benchmark in Section~\ref{sec:experiments}, this oracle is implemented by enumerating policies and extreme lower-tail envelopes and solving the remaining two-dimensional optimization geometrically.}
\section{\textcolor{black}{Regret Analysis with Known Transitions}}\label{sec:main}
\label{sec:known}
\textcolor{black}{We now provide the theoretical guarantees for WSP-CVaR-RLHF with known transition kernels.} The proof proceeds from reward confidence, to a one-episode CVaR bound, and finally to a tail-weighted summation over episodes. \textcolor{black}{Detailed proofs are provided in Appendix~\ref{app:known-proofs}.}
\begin{lemma}[Weighted confidence sequence]
\label{lem:confidence}
For every $\delta_{\rm cs}\in(0,1)$, with probability at least $1-\delta_{\rm cs}$, $\theta^\star\in\cC_k$ simultaneously for every $k\le K$.
\end{lemma}


Lemma~\ref{lem:confidence} guarantees that the confidence sets remain valid under adaptive post-label corruption. In particular, the clipping rule gives the pathwise bound \(\left\|\sum_{i<k}w_i(O_i-L_i)z_i\right\|_{\Sigma_k^{-1}} \le \chi C.\)

\begin{lemma}[Risk-envelope interface]
\label{lem:envelope}
For each episode $k\in[K]$, let
\(
\Delta_k=J_{\PP^\star,\theta^\star}^\star-J_{\PP^\star,\theta^\star}(\pi_k)
\)
be the one-episode regret.  On the confidence event, define
\(
 W_k(\tau)= \sup_{\theta,\theta'\in\cC_k} |z(\tau)^\top(\theta-\theta')|.
\)
There is a fixed, measurably tie-broken envelope $q_k$ for the true return under $\pi_k$ such that $0\le q_k\le1/\alpha$, $\E[q_k(\tau_k)\mid\mathcal M_{k-1}] =\E[q_k(\tau_k)\mid\mathcal F_{k-1}]=1$, and
\(
 \Delta_k\le \E[q_k(\tau_k)W_k(\tau_k)\mid\mathcal F_{k-1}] =\E[q_k(\tau_k)W_k(\tau_k)\mid\mathcal M_{k-1}].
\)
\end{lemma}
Lemma~\ref{lem:envelope} converts reward optimism into a CVaR-weighted confidence width.  Retaining the identity $\E[q_k\mid\mathcal F_{k-1}]=1$, rather than using only $q_k\le1/\alpha$, is necessary for the $\sqrt{K/\alpha}$ term.

Define
\[
\begin{aligned}
G_K&=d\log\left(1+\frac{\kappa K}{\lambda d}\right),
&\overline\beta&=\max_{k\le K}\beta_k,\\
S_Q&=K+\sqrt{\frac{2K\log(1/\delta_q)}{\alpha}}
+\frac{2\log(1/\delta_q)}{3\alpha}.
\end{aligned}
\]
\[
M_1=\max\{2B,2\overline\beta/\sqrt\kappa\},\qquad M_2=\max\{2B,2\overline\beta/(\chi\kappa)\}.
\]
Although $\overline\beta$ is data dependent, it has the deterministic bound
\begin{equation}
\overline\beta\le
\sqrt\lambda B+\chi C+
\frac1{\sqrt\kappa}
\sqrt{G_K+2\log(1/\delta_{\rm cs})}.
\label{eq:beta-deterministic}
\end{equation}
Here $G_K$ is the usual elliptical-potential budget.  The bound $S_Q=K+\Ot(\sqrt{K/\alpha}+1/\alpha)$, rather than $K/\alpha$, yields the clean $\sqrt{K/\alpha}$ term. Balancing the contributions $\chi C$ in $\overline\beta$ and $1/\chi$ in $M_2$ yields the corrupted $C/\alpha$ dependence.

\begin{lemma}[Tail-weighted width sum]
\label{lem:width}
Let $\delta_m,\delta_q\in(0,1)$ be the failure probabilities for the martingale-conversion and lower-tail-mass events, respectively. Define the stopped process by using $(q_k,W_k)$ from Lemma~\ref{lem:envelope} up to the first violation of $\theta^\star\in\cC_k$ and setting $(q_k,W_k)=(1,0)$ thereafter.  With probability at least $1-\delta_m-\delta_q$, this process satisfies
\[
\begin{aligned}
\sum_{k=1}^K\E[q_kW_k\mid\mathcal M_{k-1}]
&\le 2M_1\sqrt{\frac{2S_QG_K}{\alpha}}\\
&\quad+\frac{4G_KM_2}{\alpha}
+\frac{4B}{\alpha}\log\frac1{\delta_m}.
\end{aligned}
\]
Consequently, on the intersection of this martingale event and the confidence event of Lemma~\ref{lem:confidence}, the same inequality holds for the original lower-tail envelopes and widths.  By Lemma~\ref{lem:envelope}, the conditional expectations are also the learner-history conditional expectations appearing in the regret decomposition.
\end{lemma}


Let
\begin{equation}
\ifdefined\submissionversion
\mathsf{Rwd}(K):=2M_1\sqrt{\frac{2S_QG_K}{\alpha}}+\frac{4G_KM_2}{\alpha}+\frac{4B}{\alpha}\log\frac1{\delta_m}.
\else
\begin{aligned}
\mathsf{Rwd}(K)
:={}&2M_1\sqrt{\frac{2S_QG_K}{\alpha}}
+\frac{4G_KM_2}{\alpha}\\
&+\frac{4B}{\alpha}\log\frac1{\delta_m}.
\end{aligned}
\fi
\label{eq:rwd}
\end{equation}

\begin{theorem}[Known-transition regret]
\label{thm:known}
Under Assumptions~\ref{ass:model} and~\ref{ass:link} and the corruption and fresh-randomness protocol in Section~\ref{sec:model}, suppose the transition kernel $\PP^\star$ is known and the learner is given $K,\{\phi_h\}_{h=1}^H,\Theta,B,\sigma$, the risk level $\alpha$, a certified derivative lower bound $\kappa$, and a valid upper bound $C$ on the realized flip budget. Let $\lambda>0$, $\chi>0$, and suppose the augmented-state optimistic problem in Eq.~\eqref{eq:knownoracle} is solved exactly with fixed measurable tie-breaking. Let $\delta_{\rm cs},\delta_m,\delta_q\in(0,1)$ satisfy $\delta_{\rm cs}+\delta_m+\delta_q<1$. Then WSP-CVaR-RLHF satisfies, with probability at least $1-\delta_{\rm cs}-\delta_m-\delta_q$,
\(
\operatorname{Regret}(K)\le \min\!\left\{ 2BK,\; \mathsf{Rwd}(K) \right\}.
\)
For $B=O(1)$ and $C>0$, setting $\lambda=1/B^2$ and $\chi=\sqrt{G_K}/(C\sqrt\kappa)$ gives
\[
\operatorname{Regret}(K) =\Ot\!\left( \frac d\kappa\sqrt{\frac K\alpha} +\frac{dC}{\kappa\alpha}+\frac d{\kappa\alpha} +\frac B\alpha \right).
\]
For integer $C\ge1$ and $B=O(1)$, the last two terms are absorbed by the corruption term.
\end{theorem}


For $C=0$, the clipped-episode term is unnecessary. The explicit uncorrupted branch of Eq.~\eqref{eq:weight} gives the sharper statement below.

\begin{corollary}[Exact clean specialization]
\label{cor:clean}
Assume the conditions of Theorem~\ref{thm:known} with $C=0$ and $w_k=1$. Define
\(
\beta_k^0=\sqrt\lambda B+
\frac1{\sqrt\kappa}
\sqrt{2\log\frac{\det(\Sigma_k)^{1/2}}
{\det(\lambda I_d)^{1/2}\delta_{\rm cs}}},
\beta^0=\max_{k\le K}\beta_k^0,
M_0=\max\{2B,2\beta^0/\sqrt\kappa\}.
\)
With probability at least $1-\delta_{\rm cs}-\delta_m-\delta_q$,
\[
\operatorname{Regret}(K) \le\min\left\{2BK, 2M_0\sqrt{\frac{2S_QG_K}{\alpha}} +\frac{4B}{\alpha}\log\frac1{\delta_m} \right\}.
\]
For $B=O(1)$ and $\lambda=1/B^2$, this is
\[
\Ot\!\left( \frac d\kappa\sqrt{\frac K\alpha} +\frac d{\kappa\alpha}+\frac B\alpha \right).
\]
\end{corollary}


Theorem~\ref{thm:known} separates stochastic estimation error from adversarial corruption. The leading statistical term has a $1/\sqrt{\alpha}$ factor, whereas each effective flip contributes through a $1/\alpha$ factor. The algorithm uses a valid upper bound on $C$ to choose the clipping threshold, and the cap $2BK$ covers the regime in which the displayed terms exceed the trivial regret bound. 

\section{\textcolor{black}{Extension to Unknown Transition Kernels}}
\label{sec:unknown}
\textcolor{black}{We now extend the analysis to an unknown tabular transition kernel case. Preference labels may be corrupted as before, but transition observations remain clean. The reward estimator and its confidence analysis remain unchanged. The additional challenge is to control the discrepancy between the optimistic transition model used for planning and the true dynamics. Proofs of the results in this section are provided in Appendix~\ref{app:unknown-proof}.}

\subsection{\textcolor{black}{Transition Confidence Sets}}

For each episode $k\in[K]$ and stage $h\in[H]$, let $N_{k,h}(s,a)$ be the number of visits to $(h,s,a)$ before episode $k$, and let $\widehat{\PP}_{k,h}(\cdot\mid s,a)$ be the empirical next-state distribution. When the count is zero, set the empirical row to the uniform distribution on $\mathcal S$. For a transition-confidence failure probability $\delta_P\in(0,1)$, define the logarithmic factor \(L_P=S\log2+\log\frac{2HSAK}{\delta_P}\) and set $\mathrm{rad}_{k,h}^{P}(s,a)=2$ when $N_{k,h}(s,a)=0$. Otherwise, define \(\mathrm{rad}_{k,h}^{P}(s,a)=\min\{2,\sqrt{2L_P/N_{k,h}(s,a)}\}.\) Let $\cP_k$ be the rectangular set of kernels whose rows are within $L_1$ distance $\mathrm{rad}_{k,h}^{P}(s,a)$ of the empirical rows.  A visit-index union bound, using the empirical-distribution $L_1$ deviation inequality of Weissman \emph{et al.}~\cite{weissman2003l1}, gives $\PP^\star\in\cP_k$ simultaneously with probability $1-\delta_P$.

\subsection{\textcolor{black}{Unknown-Transition Algorithm}}

\textcolor{black}{The unknown-transition method replaces Eq.~\eqref{eq:knownoracle} by the joint optimistic problem}
\begin{equation}
	(\widetilde\theta_k,\widetilde b_k,\rho_k,\widetilde\PP_k)
	\in\argmax_{\substack{\theta\in\cC_k,\,b\in[-B,B],\\
			\rho\in\Pi^{\rm Aug},\,\PP\in\cP_k}}
	\left\{b-\frac1\alpha V_{1,\PP,\theta}^{\rho}(s_1,b)\right\},
	\label{eq:unknownoracle}
\end{equation}
\textcolor{black}{with fixed measurable tie-breaking. Algorithm~\ref{alg:unknown-weighted} makes the transition-estimation and update steps explicit.}
\begin{algorithm}[H]
\color{black}
	\caption{WSP-CVaR-RLHF with Unknown Transitions}
	\label{alg:unknown-weighted}
	\begin{algorithmic}[1]
		\STATE \textbf{Input:} the model and feedback protocol in Section~\ref{sec:model}; $\lambda$, $\chi$, $\delta_{\rm cs}$, $\delta_P$, and $\alpha$.
		\STATE \textbf{Initialization}: $\Sigma_1=\lambda I_d$, an empty preference stream, and zero transition counts.
		\FOR{$k=1,\ldots,K$}
		\STATE Compute $\widehat\theta_k$, $\beta_k$, and $\cC_k$ from Eqs.~\eqref{eq:mest}--\eqref{eq:conf}.
		\STATE Compute the empirical transition kernel from past transition counts and construct the rectangular confidence set $\cP_k$.
		\STATE Solve Eq.~\eqref{eq:unknownoracle} for $(\widetilde\theta_k,\widetilde b_k,\rho_k,\widetilde\PP_k)$.
		\STATE Form and execute $\pi_k=\pi^{\rho_k,\widetilde b_k,\widetilde\theta_k}$, and observe $\tau_k$ and $z_k=z(\tau_k)$.
		\STATE Before receiving the label, compute $u_k=\|z_k\|_{\Sigma_k^{-1}}$ and choose $w_k$ by Eq.~\eqref{eq:weight}.
		\STATE Receive $O_k$ and append $(z_k,O_k,w_k)$ to the preference stream.
		\STATE Update $\Sigma_{k+1}=\Sigma_k+\kappa w_kz_kz_k^\top$ and the transition counts using the observed trajectory.
		\ENDFOR
	\end{algorithmic}
\end{algorithm}

\textcolor{black}{The reward estimator, uncertainty weights, and reward confidence radius are identical to those in Algorithm~\ref{alg:weighted}. The additional steps construct $\cP_k$, solve the joint reward-transition oracle, and update the empirical transition model.}

Because the oracle now optimizes jointly over two data-dependent confidence sets, we first verify that its optimizer is well defined and measurable.

\begin{proposition}[Well-posed optimistic planning]
	\label{prop:oracle-wellposed}
	For every learner history, the maxima in Eqs.~\eqref{eq:knownoracle} and~\eqref{eq:unknownoracle} are attained. Their maximizing parameters and induced physical policies can be selected measurably with respect to the learner history.
\end{proposition}


\subsection{\textcolor{black}{Transition Error under CVaR}}

The known-transition analysis compares reward parameters under a common trajectory law. Here the optimistic kernel $\widetilde\PP_k$ differs from $\PP^\star$, so the following lemma converts this model mismatch into visitation-weighted transition errors.

\begin{lemma}[History-level CVaR simulation]
	\label{lem:simulation}
	For any fixed physical policy $\pi\in\Pi_{\mathcal H}$, any $\theta\in\Theta$, and two kernels $\PP,\QQ$, let $D_h(s,a)=\TV(\PP_h(\cdot\mid s,a),\QQ_h(\cdot\mid s,a))$, where $\TV(p,q)=\frac12\|p-q\|_1$.  Then
	\[
	|J_{\PP,\theta}(\pi)-J_{\QQ,\theta}(\pi)| \le\frac{2B}{\alpha} \sum_{h=1}^{H} \E_{\QQ,\pi}[D_h(S_h,A_h)].
	\]
	Here $(S_h,A_h)$ denotes the random stage-$h$ state-action pair under $(\QQ,\pi)$.
\end{lemma}


Lemma~\ref{lem:simulation} allows the transition error to be separated from the reward term in Lemma~\ref{lem:envelope}. The factor $2B/\alpha$ comes from shifting centered returns into $[0,2B]$ and applying the CVaR threshold representation.

For a transition-summation failure probability $\delta_t\in(0,1)$, define the cumulative transition-width bound
\begin{equation}
\ifdefined\submissionversion
	\mathsf{Tr}(K):=2HSA+4\sqrt{2L_P}\,H\sqrt{SAK}+2H\log\frac1{\delta_t}.
\else
	\begin{aligned}
		\mathsf{Tr}(K):={}&2HSA+4\sqrt{2L_P}\,H\sqrt{SAK}\\
		&+2H\log\frac1{\delta_t}.
	\end{aligned}
\fi
\label{eq:trans}
\end{equation}

Combining this transition-width bound with the known-transition reward term $\mathsf{Rwd}(K)$ yields the following result.

\begin{theorem}[Unknown-transition regret]
	\label{thm:unknown}
	Under Assumptions~\ref{ass:model} and~\ref{ass:link} and the corruption and fresh-randomness protocol in Section~\ref{sec:model}, suppose the learner is given $K$, $\{\phi_h\}_{h=1}^H$, $\Theta$, $B$, and $\sigma$, the risk level $\alpha$, a certified derivative lower bound $\kappa$, and a valid upper bound $C$ on the realized flip budget. Let $\lambda>0$, $\chi>0$, and suppose Eq.~\eqref{eq:unknownoracle} is solved exactly with fixed measurable tie-breaking. Assume physical transitions are observed cleanly and are not corrupted by the label adversary. Let $\delta_{\rm cs},\delta_m,\delta_q,\delta_P,\delta_t\in(0,1)$ have sum less than one. Here $\delta_{\rm cs}$ and $\delta_P$ are inputs to the reward and transition confidence sets, respectively. With probability at least $1-\delta_{\rm cs}-\delta_m-\delta_q-\delta_P-\delta_t$, \textcolor{black}{Algorithm~\ref{alg:unknown-weighted}} satisfies
	\[
	\operatorname{Regret}(K) \le \min\left\{ 2BK,\; \mathsf{Rwd}(K) +\frac{2B}{\alpha}\mathsf{Tr}(K) \right\},
	\]
	where $\mathsf{Rwd}(K)$ is defined in Eq.~\eqref{eq:rwd}. This finite-time bound retains the dependence on the reward scale $B$. Under the normalized reward scale $B=O(1)$ and for $C>0$, set $\lambda=1/B^2$ and $\chi=\sqrt{G_K}/(C\sqrt\kappa)$. Then
	\ifdefined\submissionversion
	\[
	\operatorname{Regret}(K)=\Ot\!\left(\frac d\kappa\sqrt{\frac K\alpha}+\frac{dC}{\kappa\alpha}+\frac d{\kappa\alpha}+\frac{H}{\alpha}\left(SA+S\sqrt{AK}\right)\right).
	\]
	\else
	\[
	\begin{aligned}
		\operatorname{Regret}(K)
		=\Ot\!\biggl(&
		\frac d\kappa\sqrt{\frac K\alpha}
		+\frac{dC}{\kappa\alpha}+\frac d{\kappa\alpha}\\
		&+\frac{H}{\alpha}\left(SA+S\sqrt{AK}\right)
		\biggr).
	\end{aligned}
	\]
	\fi
	For $C=0$, set $w_k=1$ and replace the reward-side term by the exact clean bound in Corollary~\ref{cor:clean}. The transition term remains unchanged.
\end{theorem}



Theorem~\ref{thm:unknown} adds the cost of estimating the transition model to the known-transition reward term.  The factor $S\sqrt{AK}$ follows from the Weissman $L_1$ radius because $L_P$ contains $S\log2$.  The transition term has a $1/\alpha$ factor, as illustrated by the following example.

\begin{remark}[Why transition learning is different]
	Let a one-step return equal zero with probability $p\le\alpha$ and one otherwise.  Its CVaR is $1-p/\alpha$.  For $0\le\varepsilon\le p$, reducing the bad-state probability by $\varepsilon$ changes CVaR by exactly $\varepsilon/\alpha$.  Therefore, direct transition simulation can incur $1/\alpha$ sensitivity even when the reward-estimation term has the sharper clean dependence.  
\end{remark}
\section{Numerical Experiments}
\label{sec:experiments}

\begin{figure*}[t]
\centering
\includegraphics[width=0.90\textwidth]{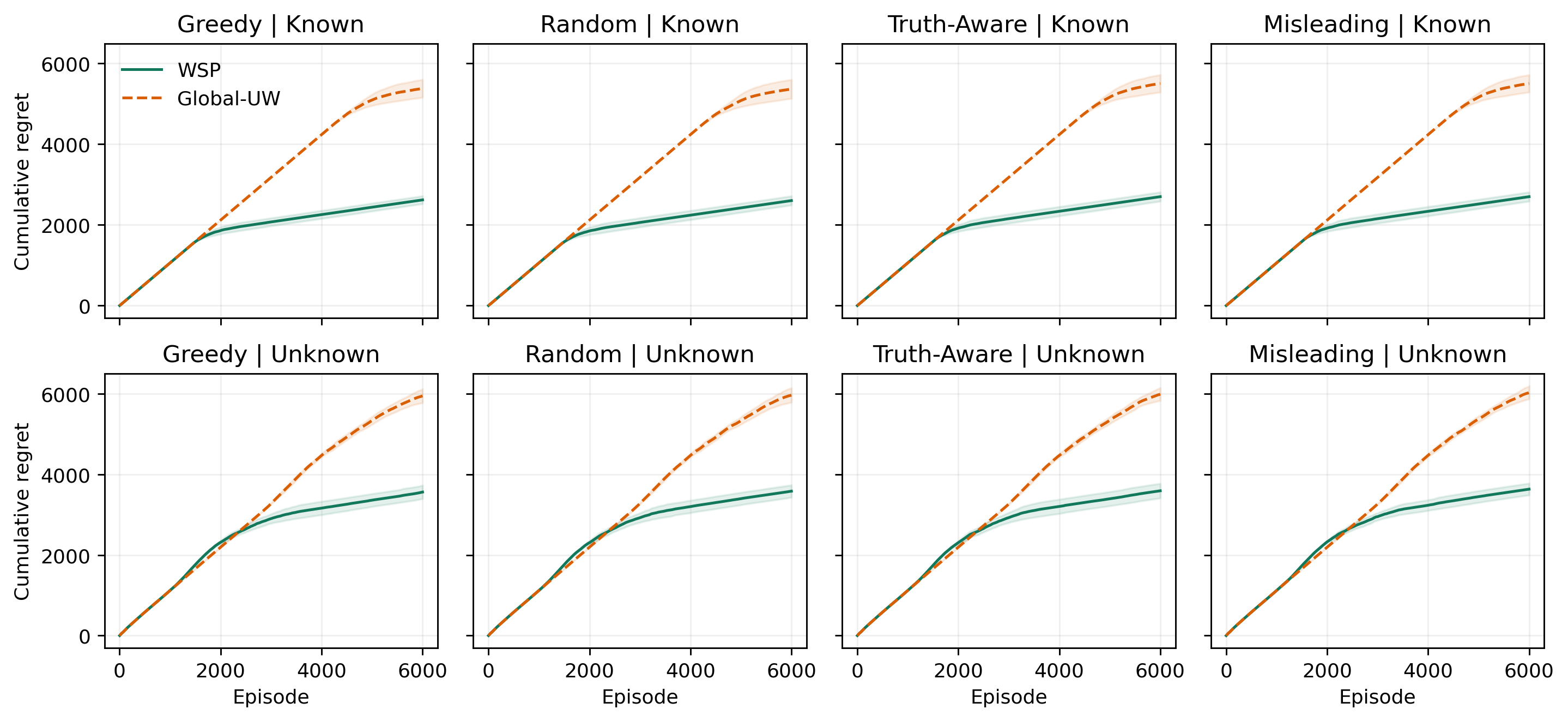}
\caption{Cumulative regret as a function of the episode index under four attacks, with $C=20$ and $\alpha=0.2$ fixed. The top and bottom rows show known and unknown transitions, respectively. Curves and shaded regions are means and 95\% confidence intervals over ten paired trials up to $K=6000$.}
\label{fig:attack-confirmation}
\end{figure*}

\begin{figure*}[t]
\centering
\includegraphics[width=0.60\textwidth]{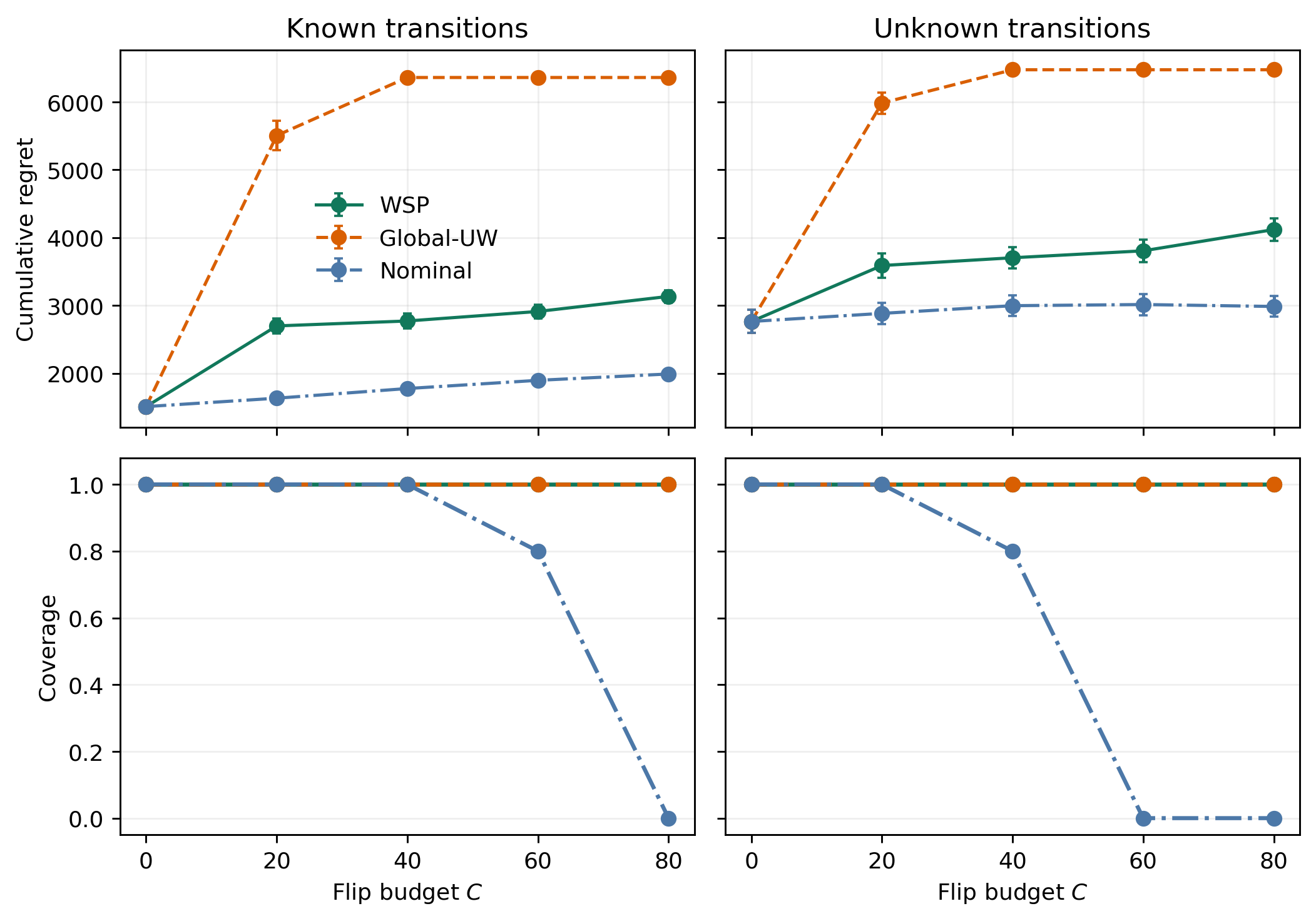}
\caption{Final cumulative regret at $K=6000$ as a function of the truth-aware flip budget $C$, with $\alpha=0.2$, over the same ten paired trials as Fig.~\ref{fig:attack-confirmation}. The top row reports regret with 95\% confidence intervals; the bottom row reports simultaneous joint coverage.}
\label{fig:budget-confirmation}
\end{figure*}

We conduct experiments to evaluate WSP-CVaR-RLHF under adversarial feedback. We adapt the synthetic attack protocol of RCDB~\cite{di2025dueling} to an $H=2$ static-CVaR MDP with additive two-dimensional rewards. At the initial state, the learner selects one of nine controllers, after which the transition kernel generates one of four terminal outcomes. The controllers have different lower-tail profiles, and the true static-CVaR optimum is separated from reward-optimistic and transition-optimistic decoys. Comparisons follow the logistic model and are corrupted by greedy, random, truth-aware, or misleading-target flips. We evaluate both a known-transition learner and an unknown-transition learner that estimates the same tabular kernel from clean transition observations.

To the best of our knowledge, no established benchmark or directly comparable online algorithm exists for static-CVaR RLHF under adversarial preference flips. We therefore compare WSP with two controlled baselines. \emph{Global-UW} sets all preference weights to one and uses the global corruption radius $C/\sqrt\lambda$, thereby isolating uncertainty weighting. The \emph{nominal} method assumes uncorrupted feedback and omits the corruption-dependent radius. All methods use the same static-CVaR planner, common random numbers, and unscaled theoretical confidence radii. A trial has simultaneous coverage if the true reward parameter remains in its confidence set throughout; in the unknown-transition case, joint coverage additionally requires all true transition rows to remain in their rectangular confidence sets.

Figure~\ref{fig:attack-confirmation} focuses on the two corruption-aware methods and compares their regret across episodes under four attacks with $C=20$ fixed. For known transitions, WSP has lower final regret than Global-UW in all ten paired trials under every attack. Its mean final regret ranges from $2602.8$ to $2699.8$, compared with $5366.7$ to $5504.6$ for Global-UW, and every paired 95\% confidence interval for the difference excludes zero. Both methods retain simultaneous coverage in every trial. As shown in the bottom row of Fig.~\ref{fig:attack-confirmation}, the same advantage persists when the transition kernel is unknown: WSP obtains mean final regret between $3558.5$ and $3630.2$, whereas Global-UW obtains between $5942.7$ and $6026.2$. On this benchmark, learning the transition kernel adds $889.7$--$977.4$ mean regret to WSP relative to its paired known-transition runs, consistent with the separate transition-estimation term in Theorem~\ref{thm:unknown}.

In contrast, Fig.~\ref{fig:budget-confirmation} fixes $K=6000$ and varies the realized flip budget. Its $C=20$ points are the endpoints of the truth-aware curves in Fig.~\ref{fig:attack-confirmation}. At $C=0$, all three methods coincide within each transition setting. For every tested $C>0$, WSP has lower regret than Global-UW in all ten paired trials for both known and unknown transitions, while both robust methods retain simultaneous coverage in all ten trials. The nominal method can have lower finite-horizon regret because its confidence set is smaller, but that set is not valid under corruption: its coverage falls to $0.8$ at $C=60$ and zero at $C=80$ for known transitions, and to $0.8$ at $C=40$ and zero for $C\ge60$ for unknown transitions. These results support the roles of uncertainty weighting and the corruption-dependent confidence radius.
\section{Conclusion}
\label{sec:conclusion}

In this work, we have developed a robust risk-sensitive RLHF framework for corrupted human feedback. We have considered additive linear rewards, one comparison per episode, and a pathwise budget of adversarial post-label flips. We have proposed WSP-CVaR-RLHF, which combines uncertainty-weighted reward estimation with optimistic augmented-state CVaR planning. Its regret analysis has established a clean $\sqrt{K/\alpha}$ term and a $C/\alpha$ corruption term. We have further extended the analysis to unknown tabular transitions, where transition estimation contributes a separate $1/\alpha$-sensitive term. Experiments under four adversarial feedback attacks have demonstrated that WSP-CVaR-RLHF consistently reduces cumulative regret relative to its unweighted robust counterpart while preserving confidence-set coverage.
%

\raggedbottom
\appendices
\section{Proofs for Known Transitions}
\label{app:known-proofs}

\subsection{Proof of Lemma~\mainref{lem:confidence}}
\label{app:confidence-proof}
\label{app:filtration}
\begin{proof}

We first formalize the filtration described in Section~\mainref{sec:model}. Conditional on $\mathcal F_{k-1}$ and the selected physical policy, let the episode-$k$ transition innovations be independent of $\mathcal M_{k-1}$ and generate the new trajectory. Define the master pre-label history
\(
\mathcal H_k^{\rm pre}=\mathcal M_{k-1}\vee\sigma_{\rm alg}(\tau_k,z_k,w_k),
\)
where $\sigma_{\rm alg}(\cdot)$ denotes the generated sigma-algebra and $\vee$ denotes the join of two sigma-algebras. The weight $w_k$ is $\mathcal H_k^{\rm pre}$-measurable and is therefore fixed before the clean label is generated. We assume
\(
L_k\mid\mathcal H_k^{\rm pre}\sim\operatorname{Ber}(\sigma(z_k^\top\theta^\star)).
\)
Let $U_k$ be a fresh private adversary seed, conditionally independent of the clean-label randomness given $\mathcal H_k^{\rm pre}$. The adversary chooses $c_k$ measurably using $\mathcal H_k^{\rm pre}\vee\sigma_{\rm alg}(L_k,U_k)$, and the next master history is
\(
\mathcal M_k=\mathcal H_k^{\rm pre}\vee\sigma_{\rm alg}(L_k,U_k,c_k,O_k).
\)

Consequently, $\mathcal M_{k-1}\subseteq\mathcal H_k^{\rm pre}\subseteq\mathcal M_k$. By the fresh-randomness convention, the trajectory law conditional on $\mathcal M_{k-1}$ is the law induced by $(\pi_k,\PP^\star)$ conditional on $\mathcal F_{k-1}$; latent past labels, flips, and adversary seeds have no additional effect. Thus the martingale arguments below use the nested master filtration, whereas the learner remains measurable with respect to $\mathcal F_{k-1}$.

The weighted score, uncertainty clipping, and clipped-score decomposition adapt the argument of Di, He, and Gu~\cite{di2025dueling} from contextual dueling bandits to trajectory-predictable comparison features.

\emph{Score decomposition.} Let
\[
\mathcal L_k(\theta) =\frac{\lambda}{2}\|\theta\|_2^2 +\sum_{i<k}w_i\bigl[\Psi(z_i^\top\theta)-O_i z_i^\top\theta\bigr].
\]
Write $\epsilon_i=L_i-\sigma(z_i^\top\theta^\star)$ and $\xi_i=O_i-L_i$.  Then the random score at the true parameter is the sum of a regularization term, a clean martingale term $\sum_{i<k}w_i\epsilon_i z_i$, and a corrupted score term $\sum_{i<k}w_i\xi_i z_i$.

\emph{Clean martingale term.} Use the nested master filtration constructed above. In particular, \(\mathcal H_i^{\rm pre}\subseteq\mathcal M_i \subseteq\mathcal H_{i+1}^{\rm pre}.\) Hence the pre-label histories form a nested label-time filtration: the $i$th design is measurable at $\mathcal H_i^{\rm pre}$, the $i$th clean noise is revealed before $\mathcal H_{i+1}^{\rm pre}$, and $\E[\epsilon_i\mid\mathcal H_i^{\rm pre}]=0$ by the fresh-label assumption. Define
\(
x_i=\sqrt{\kappa w_i}z_i, \eta_i=\sqrt{w_i/\kappa}\epsilon_i,
\), so that $x_i$ is $\mathcal H_i^{\rm pre}$-measurable and $\E[\eta_i\mid\mathcal H_i^{\rm pre}]=0$.  Since $|\epsilon_i|\le1$, $\epsilon_i$ is conditionally $1$-sub-Gaussian. Moreover, since $0\le w_i\le1$, $\eta_i$ is conditionally $1/\sqrt\kappa$-sub-Gaussian.  The rescaling also gives \(\lambda I+\sum_{i<k}x_ix_i^\top=\Sigma_k,\qquad x_i\eta_i=w_i\epsilon_i z_i.\) To match the single-index notation of the self-normalized theorem, set $\mathcal G_i=\mathcal H_{i+1}^{\rm pre}$ for $i\ge0$.  For $i\ge1$, $x_i$ is $\mathcal G_{i-1}$-measurable, $\eta_i$ is $\mathcal G_i$-measurable, and $\E[\eta_i\mid\mathcal G_{i-1}]=0$. The time-uniform vector self-normalized inequality of~\cite[Theorem~1]{abbasi2011improved}, applied to this filtration, therefore gives, simultaneously for all $k\le K$, with probability at least $1-\delta_{\rm cs}$,
\begin{align}
\left\|\sum_{i<k}w_i\epsilon_i z_i\right\|_{\Sigma_k^{-1}}
&\le\frac{1}{\sqrt\kappa}
\sqrt{2\log\frac{\det(\Sigma_k)^{1/2}}
{\det(\lambda I)^{1/2}\delta_{\rm cs}}}.
\label{eq:app-clean}
\end{align}

\emph{Corrupted score.} For the corrupted component, Eq.~\mainequation{eq:weight}, together with monotonicity of inverse covariance matrices, gives the pathwise bound
\begin{align}
\left\|\sum_{i<k}w_i\xi_i z_i\right\|_{\Sigma_k^{-1}}
&\le\sum_{i<k}c_iw_i\|z_i\|_{\Sigma_k^{-1}}\nonumber\\
&\le\sum_{i<k:c_i=1}w_i\|z_i\|_{\Sigma_i^{-1}}\nonumber\\
&\le \chi C.
\label{eq:app-corrupt}
\end{align}
The first inequality uses $|\xi_i|=|O_i-L_i|\le c_i$, and the second uses $\Sigma_k\succeq\Sigma_i$.

\emph{Parameter error.} Let $e_k=\widehat\theta_k-\theta^\star$.  Constrained first-order optimality at the minimizer, tested against the feasible point $\theta^\star$, gives $\langle\nabla\mathcal L_k(\widehat\theta_k),e_k\rangle\le0$. For every point $\theta^\star+t e_k$ on the segment between the two feasible parameters,
\[
\nabla^2\mathcal L_k(\theta^\star+t e_k) =\lambda I_d+ \sum_{i<k}w_i\sigma'(z_i^\top(\theta^\star+t e_k))z_iz_i^\top \succeq\Sigma_k,
\]
where convexity of $\Theta$ keeps the segment inside the score range covered by Assumption~\mainref{ass:link}.  Integrating this Hessian along the segment gives
\begin{align}
\|e_k\|_{\Sigma_k}^2
&\le\langle\nabla\mathcal L_k(\widehat\theta_k)
-\nabla\mathcal L_k(\theta^\star),e_k\rangle\nonumber\\
&\le-\langle\nabla\mathcal L_k(\theta^\star),e_k\rangle\nonumber\\
&\le\|\nabla\mathcal L_k(\theta^\star)\|_{\Sigma_k^{-1}}
\|e_k\|_{\Sigma_k}.
\end{align}
The score at the truth is
\[
\nabla\mathcal L_k(\theta^\star) =\lambda\theta^\star -\sum_{i<k}w_i\epsilon_i z_i -\sum_{i<k}w_i\xi_i z_i.
\]
Since $\Sigma_k\succeq\lambda I_d$, $\|\lambda\theta^\star\|_{\Sigma_k^{-1}} \le\sqrt\lambda B$.  Combining this inequality with Eqs.~\eqref{eq:app-clean} and~\eqref{eq:app-corrupt}, and then dividing by $\|e_k\|_{\Sigma_k}$ when it is nonzero, gives $\|e_k\|_{\Sigma_k}\le\beta_k$.  The self-normalized event is simultaneous in $k$, while the corruption bound is pathwise.
\end{proof}

\subsection{Proof of Lemma~\mainref{lem:envelope}}
\label{app:envelope-proof}
\begin{proof}

\emph{Optimism.} Because the horizon, state space, and action space are finite, each fixed policy induces a probability distribution on a finite trajectory set. Fix a pre-episode history on the confidence event and condition on $\mathcal F_{k-1}$. The augmented-state planner selects $(\widetilde\theta_k,\widetilde b_k,\rho_k)$ in Eq.~\mainequation{eq:knownoracle}, which induces the physical policy $\pi_k=\pi^{\rho_k,\widetilde b_k,\widetilde\theta_k}$. Since $\theta^\star\in\cC_k$, a true optimal triple $(\theta^\star,b^\star,\rho^\star)$ is feasible. Therefore, Eq.~\mainequation{eq:augmented-cvar} and optimism give
\[
\begin{aligned}
J_{\PP^\star,\theta^\star}^\star
&\le \widetilde b_k-\frac1\alpha V_{1,\PP^\star,\widetilde\theta_k}^{\rho_k}(s_1,\widetilde b_k)\le J_{\PP^\star,\widetilde\theta_k}(\pi_k).
\end{aligned}
\]
The second inequality holds because $J_{\PP^\star,\widetilde\theta_k}(\pi_k)$ re-optimizes the CVaR threshold while keeping the physical policy $\pi_k$ fixed. It remains to compare the two rewards under this same trajectory law.

\emph{Lower-tail envelope.} Let $X(\tau)=z(\tau)^\top\theta^\star$ and $Y(\tau)=z(\tau)^\top\widetilde\theta_k$ under the conditional law of $\tau_k$ induced by $(\PP^\star,\pi_k)$.  Order the distinct values in the finite support of $X$ increasingly and choose $b_k^{\rm q}$ as the first value whose cumulative conditional probability is at least $\alpha$.  Thus \(\Pp(X<b_k^{\rm q}\mid\mathcal F_{k-1}) \le\alpha\le \Pp(X\le b_k^{\rm q}\mid\mathcal F_{k-1}).\) Define
\[
q_k(\tau)=
\begin{cases}
1/\alpha,&X(\tau)<b_k^{\rm q},\\
\gamma,&X(\tau)=b_k^{\rm q},\\
0,&X(\tau)>b_k^{\rm q},
\end{cases}
\]
where, if $p_0=\Pp(X=b_k^{\rm q}\mid\mathcal F_{k-1})>0$, \(\gamma =\frac{1-\alpha^{-1}\Pp(X<b_k^{\rm q}\mid\mathcal F_{k-1})}{p_0},\) and $p_0>0$ because $b_k^{\rm q}$ is selected from the support.  On trajectories of conditional probability zero, define $q_k$ arbitrarily, say as zero.  These conventions make $b_k^{\rm q}$ and $q_k$ measurable functions of the pre-episode history.  The quantile inequalities imply $p_<:=\Pp(X<b_k^{\rm q}\mid\mathcal F_{k-1})\le\alpha$ and $\alpha-p_<\le p_0$.  Therefore \(0\le1-p_</\alpha\le p_0/\alpha,\) which proves $0\le\gamma\le1/\alpha$.  Hence $0\le q_k\le1/\alpha$ and \(\E[q_k(\tau_k)\mid\mathcal F_{k-1}]=1.\)
Moreover, $q_k$ minimizes the lower-tail risk envelope
\[
\CVaR_\alpha(X) =\inf_{0\le q\le1/\alpha,\ \E[q\mid\mathcal F_{k-1}]=1} \E[qX\mid\mathcal F_{k-1}],
\]
because it assigns the maximal allowed density to the smallest conditional values of $X$ and uses the boundary mass only as needed to make total mass one. For completeness, if an admissible envelope assigns positive mass to a larger support value while a smaller value has density below $1/\alpha$, transferring an equal amount of probability-weighted density from the larger value to the smaller one preserves $\E[q]=1$ and cannot increase $\E[qX]$.  Repeating this finite exchange produces exactly the displayed $q_k$, including its partial mass at the atom $b_k^{\rm q}$.  This proves optimality without assuming a continuous return distribution.

\emph{Dual representation.} For admissible $q$, the dual form gives $\CVaR_\alpha(Y)\le\E[qY\mid\mathcal F_{k-1}]$. Using the minimizing envelope above for $X$, \(\CVaR_\alpha(Y)-\CVaR_\alpha(X) \le \E[q_k(Y-X)\mid\mathcal F_{k-1}].\) Combining this inequality with optimism yields
\[
\Delta_k \le \E[q_k z(\tau_k)^\top(\widetilde\theta_k-\theta^\star) \mid\mathcal F_{k-1}] \le \E[q_kW_k(\tau_k)\mid\mathcal F_{k-1}],
\]
since both $\widetilde\theta_k$ and $\theta^\star$ lie in $\cC_k$. Conditional on $\mathcal M_{k-1}$, the fresh trajectory has the same law as conditional on $\mathcal F_{k-1}$: the policy is selected from the learner history, and latent past labels and adversary randomness do not enter the physical dynamics.  Hence the two conditional expectations above are unchanged when $\mathcal F_{k-1}$ is replaced by $\mathcal M_{k-1}$, and $\E[q_k\mid\mathcal M_{k-1}]=1$.
\end{proof}

\subsection{Proof of Lemma~\mainref{lem:width}}
\label{app:width-proof}
\begin{proof}

\emph{Stopped processes.} For concentration, extend the processes beyond the confidence event as follows.  At the first pre-episode history for which $\theta^\star\notin\cC_k$, stop the width process and set $W_k=0$, $Q_k=1$, and $Y_k=0$ thereafter.  Before stopping, use the fixed, measurably tie-broken envelope from Lemma~\mainref{lem:envelope}.  This makes the following processes defined on all histories.  On the simultaneous confidence event, the stopping rule never fires.  Because $\cC_k$ is learner-history measurable, the stopping decision is also $\mathcal M_{k-1}$-measurable.

Set
\[
\begin{aligned}
Q_k&=q_k(\tau_k),\qquad Y_k=Q_kW_k(\tau_k),\\
\nu_k&=\E[Y_k\mid\mathcal M_{k-1}],\qquad R=2B.
\end{aligned}
\]

\emph{Conditional-to-realized conversion.} Because $0\le Y_k\le L:=R/\alpha$, convexity gives, for $t=1/(2L)$,
\[
\begin{aligned}
\E[e^{t(\nu_k-2Y_k)}\mid\mathcal M_{k-1}]
&\le e^{t\nu_k}
\left(1-\frac{1-e^{-2t L}}{L}\nu_k\right)\\
&\le\exp\!\left(
\left[t-\frac{1-e^{-2t L}}{L}\right]\nu_k
\right)\le1.
\end{aligned}
\]
The last inequality uses $1/2\le1-e^{-1}$.  Thus
\[
\exp\!\left( \frac{1}{2L}\sum_{k=1}^n(\nu_k-2Y_k) \right),\qquad n\le K,
\]
is a nonnegative supermartingale.  Markov's inequality yields
\begin{equation}
\sum_{k=1}^K\nu_k
\le2\sum_{k=1}^KY_k+\frac{2R}{\alpha}\log\frac1{\delta_m}.
\label{eq:app-fixed-tilt}
\end{equation}

\emph{Realized lower-tail mass.} Likewise, $\E[Q_k\mid\mathcal M_{k-1}]=1$, $Q_k-1\le1/\alpha$, and, since $0\le Q_k\le1/\alpha$, \(\operatorname{Var}(Q_k\mid\mathcal M_{k-1}) \le \E[Q_k^2\mid\mathcal M_{k-1}]\le \frac1\alpha\E[Q_k\mid\mathcal M_{k-1}]=\frac1\alpha.\) The one-sided Freedman inequality~\cite{freedman1975tail} for adapted martingale-difference increments $D_k$,
\[
\Pp\!\left( \sum_{k=1}^KD_k\ge \sqrt{2vx}+\frac{2bx}{3} \right)\le e^{-x},
\]
is valid when the predictable quadratic variation is at most $v$ and the increments are at most $b$.  It follows from the standard exponent $\exp\{-t^2/[2(v+bt/3)]\}$ because $t=\sqrt{2vx}+2bx/3$ satisfies $t^2\ge2x(v+bt/3)$. Applying it with $v=K/\alpha$, $b=1/\alpha$, and $x=\log(1/\delta_q)$ gives
\[
\sum_{k=1}^K(Q_k-1) \le \sqrt{\frac{2K\log(1/\delta_q)}{\alpha}} +\frac{2\log(1/\delta_q)}{3\alpha}
\]
with probability at least $1-\delta_q$.  Equivalently,
\begin{equation}
\sum_{k=1}^KQ_k\le S_Q
\label{eq:app-envelope-mass}
\end{equation}
on that event.

\emph{Elliptical width sum.} The confidence-set diameter satisfies \(W_k(\tau_k)\le\min\{R,2\overline\beta u_k\}.\) Partition episodes into the unclipped and clipped sets $\mathcal I_{\rm u}=\{k:w_k=1\}$ and $\mathcal I_{\rm c}=\{k:w_k<1\}$, respectively.  For $k\in\mathcal I_{\rm u}$, \(W_k(\tau_k) \le M_1\min\{1,\sqrt\kappa u_k\}.\) If $\sqrt\kappa u_k\ge1$, the right-hand side is at least $R$. Otherwise, it is at least $2\overline\beta u_k$.  Therefore, the inequality covers both terms in $\min\{R,2\overline\beta u_k\}$. Weighted Cauchy--Schwarz, $Q_k\le1/\alpha$, and the elliptical-potential lemma then give
\[
\begin{aligned}
\sum_{k\in\mathcal I_{\rm u}}Q_kW_k(\tau_k)
&\le M_1\sqrt{
\left(\sum_kQ_k\right)
\left(\frac1\alpha\sum_{k\in\mathcal I_{\rm u}}
\min\{1,\kappa u_k^2\}\right)}\\
&\le M_1\sqrt{\frac{2S_QG_K}{\alpha}}.
\end{aligned}
\]
The last inequality uses the determinant update \(\frac{\det(\Sigma_{k+1})}{\det(\Sigma_k)} =1+\kappa w_ku_k^2\) and the elliptical-potential bound~\cite[Lemma~11]{abbasi2011improved}
\[
\sum_{k=1}^K \min\{1,\kappa w_ku_k^2\} \le 2\log\frac{\det(\Sigma_{K+1})}{\det(\lambda I)} \le2G_K.
\]
The first inequality follows termwise from $\min\{1,u\}\le2\log(1+u)$ and the determinant update. The second follows from the eigenvalue arithmetic--geometric mean bound proved below. On $\mathcal I_{\rm u}$, $w_k=1$. For $k\in\mathcal I_{\rm c}$, $w_k=\chi/u_k$ and \(W_k(\tau_k) \le M_2\min\{1,\kappa w_ku_k^2\}.\) Here $w_k=\chi/u_k$.  If $\kappa\chi u_k\ge1$, the right-hand side is at least $R$. Otherwise, it is at least $(2\overline\beta/(\chi\kappa))\kappa\chi u_k =2\overline\beta u_k$. Therefore, \(\sum_{k\in\mathcal I_{\rm c}}Q_kW_k(\tau_k) \le\frac{2G_KM_2}{\alpha}.\)

Adding the unclipped and clipped sums and substituting the result into Eq.~\eqref{eq:app-fixed-tilt} proves the claim.
\end{proof}

\subsection{Proof of Theorem~\mainref{thm:known}}
\label{app:known-theorem-proof}
\begin{proof}

On the confidence event, Lemmas~\mainref{lem:envelope} and~\mainref{lem:width} give the $\mathsf{Rwd}(K)$ part of the theorem.  The bound $|z(\tau)^\top\theta^\star|\le B$ separately gives $\operatorname{Regret}(K)\le2BK$, yielding the minimum of the two bounds. It remains to verify the $\widetilde O$ display without hiding polynomial factors.

Because $\|z_k\|_2\le1$ and $0\le w_k\le1$, \(\operatorname{tr}(\Sigma_{K+1})\le \lambda d+\kappa K.\) The arithmetic--geometric mean inequality for the eigenvalues gives \(\log\frac{\det(\Sigma_{K+1})}{\det(\lambda I_d)} \le d\log\left(1+\frac{\kappa K}{\lambda d}\right)=G_K.\) Since $\Sigma_k\preceq\Sigma_{K+1}$, Eq.~\mainequation{eq:radius} implies \(\overline\beta \le\sqrt\lambda B+\chi C+ \frac1{\sqrt\kappa} \sqrt{G_K+2\log(1/\delta_{\rm cs})},\) which is Eq.~\mainequation{eq:beta-deterministic}.

Writing $\ell_q=\log(1/\delta_q)$ and using $\sqrt{2K\ell_q/\alpha}\le K/2+\ell_q/\alpha$ gives \(S_Q\le2K+\frac{2\ell_q}{\alpha},\qquad \sqrt{S_Q}\le\sqrt{2K}+\sqrt{\frac{2\ell_q}{\alpha}}.\)

For $C>0$, set $\lambda=1/B^2$ and $\chi=\sqrt{G_K}/(C\sqrt\kappa)$.  Then \(\overline\beta =\widetilde O\left( 1+\frac{\sqrt{G_K}+1}{\sqrt\kappa}\right),\) and hence
\[
\begin{aligned}
M_1&=\widetilde O\left(
B+\frac{\sqrt{G_K}+1}{\kappa}\right),\\
G_KM_2&=\widetilde O\left(
BG_K+\frac{C(G_K+1)}{\kappa}\right).
\end{aligned}
\]
Here we used $\sqrt{G_K}\le G_K+1$ and $\kappa\le1$.  Substituting these bounds and the estimate for $\sqrt{S_Q}$ into Eq.~\mainequation{eq:rwd}, then using $G_K=\widetilde O(d)$, gives
\[
\mathsf{Rwd}(K) =\widetilde O\left( \frac d\kappa\sqrt{\frac K\alpha} +\frac{dC}{\kappa\alpha}+\frac d{\kappa\alpha} +\frac B\alpha \right).
\]
\end{proof}

\subsection{Proof of Corollary~\mainref{cor:clean}}
\label{app:clean-proof}
\begin{proof}
When $C=0$, Eq.~\mainequation{eq:weight} sets $w_k=1$ for every episode.  Therefore the clipped set $\mathcal I_{\rm c}$ in the preceding proof is empty.  Repeating the unclipped Cauchy--Schwarz calculation with $\beta^0$ and $M_0$ gives \(\sum_{k=1}^KY_k \le M_0\sqrt{\frac{2S_QG_K}{\alpha}}.\) Equation~\eqref{eq:app-fixed-tilt} proves the exact display in Corollary~\mainref{cor:clean}.  With $\lambda=1/B^2$, \(\beta^0=\widetilde O\left( 1+\frac{\sqrt{G_K}+1}{\sqrt\kappa}\right),\) and the same substitutions, now without $M_2$, give
\[
\operatorname{Regret}(K) =\widetilde O\left( \frac d\kappa\sqrt{\frac K\alpha} +\frac d{\kappa\alpha}+\frac B\alpha \right).
\]
\end{proof}
\section{Proofs for Unknown Transitions}
\label{app:unknown-proof}

\subsection{Proof of Proposition~\mainref{prop:oracle-wellposed}}
\label{app:oracle-wellposed}
\begin{proof}

For fixed $(\theta,\PP)$, define the optimal augmented-state shortfall recursively by \(G_{H+1}^{\theta,\PP}(s,b)=b^+,\) and, for $h=H,\ldots,1$,
\begin{equation}
G_h^{\theta,\PP}(s,b)
=\min_{a\in\mathcal A}\sum_{s'\in\mathcal S}
\PP_h(s'\mid s,a)
G_{h+1}^{\theta,\PP}
\!\left(s',b-r_{\theta,h}^{\rm ctr}(s,a)\right).
\label{eq:oracle-dp}
\end{equation}
Finite-action backward induction gives $G_1^{\theta,\PP}(s_1,b)=\min_{\rho\in\Pi^{\rm Aug}}V_{1,\PP,\theta}^{\rho}(s_1,b)$. The terminal function is continuous in $b$. If $G_{h+1}^{\theta,\PP}$ is jointly continuous in $(\theta,b,\PP)$, then each action value in Eq.~\eqref{eq:oracle-dp} is continuous because the reward shift is linear in $\theta$ and the next-state sum is finite. The minimum over the finite action set preserves continuity. Induction therefore proves that $G_1^{\theta,\PP}(s_1,b)$ is continuous in $(\theta,b,\PP)$.

The observable histories form a standard Borel space. The estimator is uniquely defined by its strongly convex objective on the compact convex set $\Theta$ and is measurable in the observed history. Consequently, $\cC_k$ is a nonempty measurable compact correspondence because it is the intersection of compact $\Theta$ with a closed ellipsoid whose center and shape matrix are history measurable. The rectangular set $\cP_k$ is likewise a nonempty measurable compact correspondence of the empirical transition rows and visit counts; both correspondences have measurable graphs because their defining continuous inequalities have history-measurable coefficients.

After optimizing out $\rho$ through Eq.~\eqref{eq:oracle-dp}, the known- and unknown-transition objectives are Carath\'eodory functions: they are measurable in the learner history and continuous in the finite-dimensional decision variables. Continuity of the objective and compactness of the feasible set ensure that the maximum is attained. The measurable maximum theorem~\cite[Theorem~18.19]{aliprantis2006infinite} then yields a learner-history-measurable maximizing selector for $(\theta,b)$ in Eq.~\mainequation{eq:knownoracle} and for $(\theta,b,\PP)$ in Eq.~\mainequation{eq:unknownoracle}. Finally, selecting the first minimizing action in a fixed ordering of $\mathcal A$ in Eq.~\eqref{eq:oracle-dp} gives Borel-measurable augmented-state decision rules. Their induced physical policies are therefore learner-history measurable. The known-transition oracle follows as the special case $\PP=\PP^\star$.
\end{proof}

\subsection{Simultaneous Transition Confidence}

For each $h\in[H]$ and each $(s,a)$, pre-sample an infinite i.i.d. stack from $\PP_h^\star(\cdot\mid s,a)$ and reveal its $n$th element on the $n$th visit. The episode index of that visit may be adaptive, but the stack index is not. For every deterministic $n\ge1$, let $p$ denote the true next-state row, $\widehat p_n$ its empirical distribution after $n$ samples, and $\mathcal E_{n,h,s,a}$ the event that $\widehat p_n$ is farther than $\sqrt{2L_P/n}$ from $p$ in $L_1$ distance.  The empirical-distribution $L_1$ deviation bound of Weissman \emph{et al.}~\cite{weissman2003l1}, \(\Pp\!\left(\|\widehat p_n-p\|_1\ge\varepsilon\right) \le(2^S-2)e^{-n\varepsilon^2/2},\) with $\varepsilon=\sqrt{2L_P/n}$ gives \(\Pp(\mathcal E_{n,h,s,a}) \le\frac{\delta_P}{HSAK}.\)
Here $L_P$ is defined in the unknown-transition section. A union bound over all $n\le K$ and all $(h,s,a)$ proves simultaneous inclusion for every positive visit count. When the count is zero, the radius is two, so the $L_1$ ball around the fixed uniform empirical row contains every probability row. Hence $\PP^\star\in\cP_k$ simultaneously for every $k$.

\subsection{Proof of Lemma~\mainref{lem:simulation}}
\begin{proof}
Write $X_\theta(\tau)=z(\tau)^\top\theta\in[-B,B]$ and $U_\theta=X_\theta+B\in[0,2B]$.  Translation equivariance reduces the desired bound to $U_\theta$.  For a random variable in $[0,2B]$, the maximization in Eq.~\mainequation{eq:cvar} may be restricted to $b\in[0,2B]$: a threshold below zero is dominated by $b=0$, while for $b>2B$ the objective is nonincreasing in $b$ when $\alpha<1$ and constant when $\alpha=1$.

Fix $b\in[0,2B]$ and let $f_b(\tau)=(b-U_\theta(\tau))^+\in[0,2B]$. Regard a fixed physical policy as the collection of conditional action kernels it induces on observed prefixes. For $0\le j\le H$, define the hybrid law $\mathbb M^j$ to use $\QQ$ at steps $1,\ldots,j$ and $\PP$ at steps $j+1,\ldots,H$, with the same fixed initial state and physical policy in every hybrid. Thus $\mathbb M^0$ is the $\PP$-trajectory law and $\mathbb M^H$ is the $\QQ$-trajectory law.

For each $h\in[H]$, the adjacent laws $\mathbb M^{h-1}$ and $\mathbb M^h$ have the same distribution over the prefix through $(S_h,A_h)$. This is the law obtained using the $\QQ$-kernel through step $h-1$. Conditional on a fixed prefix $\mathfrak p_h=(S_1,A_1,\ldots,S_h,A_h)$, define \(g_{b,h}^{\mathfrak p_h}(s') =\E[f_b(\tau)\mid \mathfrak p_h,S_{h+1}=s'],\) where the expectation uses the common continuation policy and the common $\PP$-kernel from step $h+1$ onward.  Since $0\le g_{b,h}^{\mathfrak p_h}\le2B$, the variational definition of total variation gives
\[
\left| \E_{\PP_h(\cdot\mid S_h,A_h)}g_{b,h}^{\mathfrak p_h} -\E_{\QQ_h(\cdot\mid S_h,A_h)}g_{b,h}^{\mathfrak p_h} \right| \le2B D_h(S_h,A_h).
\]
Average over the common prefix law and telescope from $h=1$ to $H$ to obtain \(|\E_{\PP,\pi}f_b-\E_{\QQ,\pi}f_b| \le2B\sum_{h=1}^{H}\E_{\QQ,\pi}[D_h(S_h,A_h)].\) The right-hand side is uniform in $b$.  Multiplying by $1/\alpha$ in Eq.~\mainequation{eq:cvar} and using $|\max_bF(b)-\max_bG(b)|\le\sup_b|F(b)-G(b)|$ proves the claim.
\end{proof}

\subsection{Proof of Theorem~\mainref{thm:unknown}}
\begin{proof}
Consider the simultaneous reward and transition confidence event established above. On this event, $(\theta^\star,b^\star,\rho^\star,\PP^\star)$ is feasible in Eq.~\mainequation{eq:unknownoracle}. The selected triple induces the physical policy $\pi_k=\pi^{\rho_k,\widetilde b_k,\widetilde\theta_k}$, which is fixed by the pre-episode history. Therefore, optimism gives
\[
\begin{aligned}
J_{\PP^\star,\theta^\star}^\star
&\le \widetilde b_k-\frac1\alpha V_{1,\widetilde\PP_k,\widetilde\theta_k}^{\rho_k}(s_1,\widetilde b_k)\le J_{\widetilde{\PP}_k,\widetilde\theta_k}(\pi_k).
\end{aligned}
\]
The second inequality re-optimizes the CVaR threshold for the fixed physical policy $\pi_k$ under $(\widetilde\PP_k,\widetilde\theta_k)$.
Subtract $J_{\PP^\star,\theta^\star}(\pi_k)$ and split at $J_{\PP^\star,\widetilde\theta_k}(\pi_k)$:
\[
\begin{aligned}
\Delta_k
&\le J_{\widetilde{\PP}_k,\widetilde\theta_k}(\pi_k)
-J_{\PP^\star,\theta^\star}(\pi_k)\\
&=\left[J_{\PP^\star,\widetilde\theta_k}(\pi_k)
-J_{\PP^\star,\theta^\star}(\pi_k)\right]\\
&\quad+\left[J_{\widetilde{\PP}_k,\widetilde\theta_k}(\pi_k)
-J_{\PP^\star,\widetilde\theta_k}(\pi_k)\right].
\end{aligned}
\]
For the first bracket, construct the lower-tail envelope $q_k$ for the true return under $(\PP^\star,\pi_k)$.  The lower-tail-envelope and dual-representation arguments in the proof of Lemma~\mainref{lem:envelope} give \(J_{\PP^\star,\widetilde\theta_k}(\pi_k) -J_{\PP^\star,\theta^\star}(\pi_k) \le \E[q_kW_k\mid\mathcal M_{k-1}].\) Therefore, the remaining summation is the tail-weighted width bound of Lemma~\mainref{lem:width}.  For the second bracket, apply Lemma~\mainref{lem:simulation} with $\PP=\widetilde{\PP}_k$ and $\QQ=\PP^\star$.  Its expectation is then taken under the true realized-episode law generated by $\PP^\star$ and $\pi_k$.  The row-wise TV term satisfies
\[
\begin{aligned}
&\TV(\widetilde{\PP}_{k,h}(\cdot\mid s,a),\PP_h^\star(\cdot\mid s,a))\\
&\quad\le\frac12\|\widetilde{\PP}_{k,h}(\cdot\mid s,a)
-\PP_h^\star(\cdot\mid s,a)\|_1\\
&\quad\le\min\{1,\mathrm{rad}_{k,h}^{P}(s,a)\}.
\end{aligned}
\]
The final inequality uses that both $\PP^\star$ and $\widetilde{\PP}_k$ lie in the same rectangular confidence set, so their $L_1$ distance is at most $2\mathrm{rad}_{k,h}^{P}(s,a)$, and total variation is one half of $L_1$ distance.

\emph{Transition-width sum.} Let $(S_{k,h},A_{k,h})$ denote the realized state-action pair at stage $h$ of episode $k$.  Set \(\Gamma_k^{\rm tr}=\sum_{h=1}^{H} \min\{1,\mathrm{rad}_{k,h}^{P}(S_{k,h},A_{k,h})\}.\) Define its predictable mean
\[
\begin{aligned}
\mu_k^{\rm tr}&=\E[\Gamma_k^{\rm tr}\mid\mathcal M_{k-1}]\\
&=\sum_{h=1}^{H}\E_{\PP^\star,\pi_k}\!\left[
\min\{1,\mathrm{rad}_{k,h}^{P}(S_h,A_h)\}
\mid\mathcal M_{k-1}\right].
\end{aligned}
\]
The equality follows because $\pi_k$ and every radius are fixed by the pre-episode learner history, while the new trajectory follows the fresh true kernel.  Consequently, the conditional simulation contribution is at most $(2B/\alpha)\mu_k^{\rm tr}$. Since $0\le \Gamma_k^{\rm tr}\le H$, the same fixed-tilt exponential-supermartingale argument as in Eq.~\eqref{eq:app-fixed-tilt} gives \(\sum_{k=1}^K\mu_k^{\rm tr} \le2\sum_{k=1}^K\Gamma_k^{\rm tr}+2H\log\frac1{\delta_t}.\) Pathwise grouping by $(h,s,a)$ and $\sum_{n=1}^Nn^{-1/2}\le2\sqrt N$ yield the following calculation.  If a row is visited $N_{h,s,a}$ times, its first visit contributes at most one and the remaining visits contribute at most $2\sqrt{2L_PN_{h,s,a}}$.  Since $\sum_{h,s,a}N_{h,s,a}=HK$, Cauchy--Schwarz gives
\[
\begin{aligned}
\sum_{k=1}^K\Gamma_k^{\rm tr}
&\le HSA+2\sqrt{2L_P}
\sum_{h,s,a}\sqrt{N_{h,s,a}}\\
&\le HSA+2\sqrt{2L_P}\,H\sqrt{SAK}.
\end{aligned}
\]

By Lemmas~\mainref{lem:envelope} and~\mainref{lem:width}, the reward difference contributes the same reward-side term as Theorem~\mainref{thm:known}. Lemma~\mainref{lem:simulation} and the transition-width calculation contribute $(2B/\alpha)\mathsf{Tr}(K)$. A union bound over the reward, envelope-mass, martingale-conversion, transition-confidence, and transition-summation events proves the stated finite-time bound. Finally, because $L_P=S\log2+\log(2HSAK/\delta_P)$, \(\sqrt{L_P}\sqrt{SAK} =\widetilde O\left(S\sqrt{AK}\right).\) Consequently, the exact transition contribution simplifies only to \(\frac{BH}{\alpha}\, \widetilde O\left(SA+S\sqrt{AK}\right),\) which proves the state dependence displayed in Theorem~\mainref{thm:unknown}.
\end{proof}

\bibliographystyle{IEEEtran}
\bibliography{IEEE}

\end{document}